\documentclass[11pt,UTF8]{article}

\usepackage[UKenglish]{babel}
\usepackage{comment}
\usepackage{amsmath}
\usepackage{latexsym}
\usepackage{amssymb}
\usepackage{esvect}
\usepackage{float}
\usepackage{proof}
\usepackage{graphicx}
\usepackage{color}
\usepackage{hyperref}
\usepackage{url}
\usepackage{longtable}
\usepackage{supertabular}  
\usepackage{enumerate}
\usepackage{authblk}
\usepackage{charter}
\hypersetup{
   colorlinks,%
   citecolor=blue,%
   filecolor=black,%
   linkcolor=red,%
   urlcolor=black
 }

\usepackage[all]{xy}

\newcommand{\lr}[1]{\langle #1 \rangle}

\newcommand{\BP}{\ensuremath{\mathbf{P}}}

\newcommand{\N}{\mathcal{N}}
\newcommand{\A}{\mathbf{I}}

\newcommand{\ELKy}{\mathbf{ELKy}}
\newcommand{\KW}{\mathcal{K}\!y}
\newcommand{\K}{\mathcal{K}}
\newcommand{\KWi}{{\KW}_i}
\newcommand{\KWj}{{\KW}_j}
\newcommand{\Ki}{\K_i}
\newcommand{\SKy}{\mathbb{SKY}}
\newcommand{\SKyi}{\mathbb{SKYI}}
\newcommand{\LP}{\mathbf{LP}}

\newcommand{\E}{\mathcal{E}}

\newcommand{\M}{\mathcal{M}}

\newcommand{\gff}{\lr{\Gamma, F,\vv{f}}}
\newcommand{\dgg}{\lr{\Delta, G,\vv{g}}}
\newcommand{\thhh}{\lr{\Theta, H,\vv{h}}}
\newcommand{\jst}{\!\!:\!}

\newcommand{\bbc}{\mathbb{C}}

\newcommand{\abs}[1]{\vert #1 \vert}

\newcommand{\MP}{\texttt{MP}}
\newcommand{\DISTK}{\texttt{DISTK}}
\newcommand{\DISTY}{\texttt{DISTY}}
\newcommand{\NECK}{\texttt{NECK}}
\newcommand{\NECY}{\texttt{NECKY}}
\newcommand{\TAUT}{\texttt{TAUT}}
\newcommand{\transK}{\texttt{4}}
\newcommand{\eucK}{\texttt{5}}
\newcommand{\transY}{\texttt{4Y}}
\newcommand{\eucY}{\texttt{5Y}}
\newcommand{\transKY}{\texttt{4KY}}
\newcommand{\eucKY}{\texttt{5KY}}
\newcommand{\transYK}{\texttt{4YK}}
\newcommand{\eucYK}{\texttt{5YK}}
\newcommand{\axT}{\texttt{T}}

\newcommand{\PRES}{\texttt{PRES}}

\newtheorem{thm}{Theorem}
\newtheorem{defn}[thm]{Definition}
\newtheorem{prop}[thm]{Proposition}

\newtheorem{corol}[thm]{Corollary}
\newtheorem{lem}[thm]{Lemma}
\newtheorem{remark}[thm]{Remark}

\newenvironment{proof} {\textsc{Proof}\quad} {\hfill $\Box$\\}

\renewcommand{\phi}{\varphi}

\newcommand{\JL}{\mathbf{JL}}
\newcommand{\vDashJ}{\vDash^\mathsf{J}}
\newcommand{\MJ}{\M^\mathsf{J}}
\newcommand{\EJ}{\E^\mathsf{J}}

\begin{document}
\title{A Logic of Knowing Why}

\author[1]{Chao Xu}
\author[1]{Yanjing Wang\thanks{corresponding author}}
\author[2]{Thomas Studer }

\affil[1]{Department of Philosophy, Peking University }
\affil[2]{Institut f\"ur Informatik, University of Bern }


\date{}


\maketitle

\begin{abstract}
When we say ``I know why he was late", we know not only the fact that he was late, but also an explanation of this fact. We propose a logical framework of ``knowing why"  inspired by the existing formal studies on why-questions, scientific explanation, and justification logic. We introduce the $\KWi$ operator into the language of epistemic logic to express ``agent $i$ knows why $\varphi$'' and propose a Kripke-style semantics of such expressions in terms of knowing an explanation of $\varphi$. We obtain two sound and complete axiomatizations w.r.t. two different model classes depending on different assumptions about introspection. 
\medskip

\end{abstract}

\section{Introduction}
Ever since the seminal work by Hintikka \cite{hintikka1962knowledge}, epistemic logic has grown into a major subfield of philosophical logic, which has unexpected applications in other fields such as computer science, AI, and game theory (cf.~the handbook \cite{HBEL}).  Standard epistemic logic focuses on propositional knowledge expressed by ``knowing that $\varphi$''. However, there are various knowledge expressions in terms of ``knowing whether'', ``knowing what'', ``knowing how'', and so on, which have attracted a growing interest in recent years (cf.~the survey \cite{Wang16}).

Among those ``knowing-wh'',\footnote{``Wh'' stands for the  wh-question words.} ``knowing why'' is perhaps the most important driving force behind our advances in  understanding the world and each other. For example, we may want to know why (\cite{bird98}): 
\begin{itemize}
\item the window is broken.
\item the lump of potassium dissolved.
\item he stayed in the caf\'e all day.
\item cheetahs can run at high speeds.
\item blood circulates in the body.
\end{itemize}
Intuitively, each ``knowing why'' expression corresponds to an embedded why-question. To some extent, the process of knowing the world is to answer why-questions about the world \cite{hintikka1981logic}. In fact, there is a very general connection between knowledge and wh-questions discovered by Hinttika in the framework of quantified epistemic logic \cite{Hintikka1983}. For example, consider the question $Q:$ ``Who murdered Mary?'':
\begin{itemize}
\item The \textit{presupposition} of $Q$ is that the questioner knows that Mary was murdered by someone, formalized by $\K \exists x M(x, \text{Mary})$.
\item The \textit{desideratum} of $Q$ is that the questioner knows who murdered Mary, which is formalized by $\exists x\K M(x, \text{Mary})$. The distinction between the desideratum and the presupposition highlights the difference between \textit{de re} and \textit{de dicto} readings of knowing who. 
\item One possible answer to $Q$ is ``John murdered Mary'' formalized as $M(\text{John, Mary})$. However, telling the questioner this fact may not be enough to let the questioner know who murdered Mary since he or she may not have any idea on who John is. Therefore Hintikka also requires the following extra condition. 
\item\textit{Conclusiveness} of the above answer requires that the questioner also knows who John is ($\exists x \K(\text{John}=x))$. Conclusive answers realize the desideratum. 
\end{itemize}
However, Hintikka viewed why-questions, such as Q: ``Why $\varphi$ is the case?'', as a special degenerated case where the presupposition and desideratum are the same: 
\begin{itemize}
\item The \textit{presupposition} of $Q$ is $\K \varphi$; 
\item The \textit{desideratum} of $Q$  is $\K\varphi $.
\end{itemize}
Hintikka then developed a different logical theory of why-questions in \cite{hintikka1995semantics} using the inquiry model and the interpolation theorem of first-order logic. However, we do not think why-questions are special if we can quantify over the possible answers to them. Intuitively, an answer to a question ``Why $\varphi$?'' is an explanation of the fact $\varphi$. In this paper, we take the view shared by Koura  \cite{koura1988approach} and Schurz \cite{schurz2005explanations}:
\begin{itemize}
\item The \textit{presupposition} of $Q$ is that the questioner knows that there is an explanation for the fact $\varphi$: $\K\exists x E(x,\varphi)$.
\item The \textit{desideratum} of $Q$ is that the questioner knows why $\varphi$: $\exists x\K E(x, \varphi)$. 
\end{itemize}
Note that if explanations are \textit{factive} $\exists x E(x,\varphi)\to\varphi$, then the presupposition $\K \exists x E(x, \varphi)$ also implies $\K\varphi$ in standard quantified (normal) modal logic. 
\medskip

Now we have a preliminary logical form of knowing why in terms of the desideratum $\exists x\K E(x, \varphi)$ of the corresponding why-question. The next questions are:
\begin{enumerate}
\item What are (good) \textit{explanations}?
\item How can we capture the relation ($E$ above) between an explanation and a proposition in logic? 
\end{enumerate}

The two questions are clearly related. To answer the first one, let us look back at the examples we mentioned at the beginning of this introduction. In fact there are different kinds of explanations \cite{bird98}: 
\begin{itemize}
\item Causal: The window broke because the stone was thrown at it.
\item Nomic:\footnote{Nomic explanations are explanation in terms of laws of nature.} The lump of potassium dissolved since as a law of nature potassium reacts
with water to form a soluble hydroxide.
\item Psychological: He stayed in the caf\'e all day hoping to see her again.
\item Darwinian: Cheetahs can run at high speeds because of the selective advantage this gives them in
catching their prey.
\item Functional: Blood circulates in order to supply the various parts of the body with oxygen and
nutrients.
\end{itemize}

In philosophy of science, the emphasis is on \textit{scientific explanations} to why questions, which mainly involve Nomic and Causal explanations in the above categorization \cite{bromberger1966questions,koura1988approach,van1980scientific}.  According to Schurz \cite{Schurz1999} there are three major paradigms in understanding (scientific) explanations:\footnote{There are also various dimensions of each paradigm, e.g, probabilistic vs. non-probabilistic, singular events or laws to be explained.} 
\begin{itemize}
\item The \textit{nomic expectability approach} initiated by Hempel \cite{hempel1965aspects}, where a good explanation to $\varphi$ should make the explanandum $\varphi$ predictable or increases $\varphi$'s expectability. 
\item The \textit{causality approach} (cf. e.g, \cite{salmon1984scientific}), where an explanation to $\varphi$ should give a complete list of causes or relevant factors to $\varphi$.
\item The {unification approach} (cf. e.g., \cite{kitcher1981explanatory}) where the focus is on the global feature of explanations in a coherent picture. 
\end{itemize}
 Our initial inspiration comes from the deductive-nomological model proposed by Hempel \cite{hempel1948studies} in the first approach mentioned above, which is the mostly discussed (and criticized) model of explanation. The basic idea is that an explanation is a \textit{derivation} of the explanandum from some universally quantified laws and some singular sentences. Although such a logical empiricistic approach arouse  debates for decades,\footnote{See  \cite{Schurz1995,weber2013scientific} for critical surveys.} it draws our attention to the inner structure of explanations and its similarity to derivations in logic. In this paper, as the first step towards a logic of knowing why, we would like to stay neutral on different types of explanations and their models, and focus on the most abstract logical structure of (scientific) explanations. From a structuralist point of view, we only need to know how explanations compose and interact with each other without saying what they are exactly. 

Now, as for the second question, how can we capture the explanatory relation between explanations and propositions in logic? Our next crucial inspiration came from \textit{Justification Logic} proposed by Artemov \cite{artemov2008logic}. Aiming at making up the gap between  epistemic logic and the mainstream epistemology where justified true belief is the necessary basis of knowledge, justification logics are introduced based on the ideas of \textit{Logic of Proof} (\texttt{LP}) \cite{Artemov95,Art01BSL}.\footnote{ \texttt{LP}  was invented to give an arithmetic semantics to intuitionistic logic under the Brouwer-Heyting-Kolmogorov provability interpretation.} Justification logic introduces formulas in the shape of $t\jst\varphi$ into the logical language, read as ``$t$ is a justification of $\varphi$''.\footnote{Similar ideas also appeared in \cite{vB91} by van Benthem.} Therefore, in justification logic we can talk about knowledge with an \textit{explicit} justification. Moreover, justifications can be composed using various operations. For example, $t\jst(\varphi\to\psi)\land s\jst \varphi\to (t\cdot s)\jst \psi $ is an axiom in the standard justification logic
where $\cdot$ is the application operation of two justifications. Note that if we read $t\jst\varphi$ as ``$t$ is an explanation of the fact $\varphi$'', then this axiom also makes sense in general. 

On the other hand, conceptually, justifications are quite different from explanations. For example, the fact that the shadow of a flagpole is $x$ meters long may  justify that the length of the pole is $y$ meters given the specific time and location on earth. However, the length of the shadow of a flagpole clearly does not explain why the pole is $y$ meters long, if we are looking for causal explanations. In general, a justification of $\varphi$ gives a reason to \textit{believing} $\varphi$ (though not necessarily true), but an explanation gives a reason to \textit{being} $\varphi$, presupposing the truth of $\varphi$. In this paper, we only make use of some technical apparatus of justification logic, and there are quite some differences in our framework compared to justification logic, which will be discussed in Section \ref{sec.comp}. 

Putting all the above ideas together, we are almost ready to lay out the basis of our logic of knowing why. Following  \cite{Wang16}, we enrich the standard (multi-agent) epistemic language with a new ``knowing why'' operator $\KWi$, instead of using a quantified modal language. Roughly speaking, $\KWi\varphi$ is essentially $\exists t{\K_i ( t\jst\varphi)}$, although we do not allow quantifiers and terms in the logical language. As in \cite{Wang15:lori,WF14}, this will help us to control the expressive power of the logic in hopes of a computationally well-behaved logic. The semantics is based on the idea of Fitting model for justification logic. 

Since the language has both the standard epistemic operator $\K_i$ and also the new ``knowing why'' modality $\KWi$, there are lots of interesting things that can be expressed. For example, 
\begin{itemize}
\item $\K_ip\land \neg \KWi p$, e.g, I know that Fermat's last theorem is true but I do not know why.
\item $\neg \KWi p\land \K_i\KWj p$, e.g., I do not know why Fermat's last theorem holds but I know that Andrew Wiles knows why.
\item $\K_i\K_jp\land \neg \KWi \K_j p$, e.g., I know that you know that the paper has been accepted, but I do not know why you know. 
\item $\KW_i\K_jp\land \K_i\neg \KW_j p$, e.g., I know why you know that the paper has been rejected, but I am sure you do not know why.   
\end{itemize}
As we will see later, these situations are all satisfiable in our models.\footnote{According to our semantics to be introduced later, it is also allowed to know why for different reasons (for different people),  which can help to model mutual misunderstanding. }


\medskip

Before going into the technical details, it is helpful to summarize the aforementioned ideas: 
\begin{itemize}
\item The language is inspired by the treatments of the logics of ``knowing what'', and ``knowing how'', where new modalities of such constructions are introduced, without using the full language of quantified epistemic logic.
\item The formal treatment of explanations is inspired by the formal account of justifications in justification logics.  
\item The semantics of $\KWi$ is inspired by Hintikka's logical formulation of the desideratum of Wh-questions: $\exists t{\K_i ( t\jst\varphi)}$. 
\end{itemize}

In the rest of the paper, Section \ref{sec.basic} lays out the  language, semantics and two proof systems of our knowing why logic; Section \ref{sec.SC} proves the completeness of the two systems, and gives an alternative semantics closer to the standard justification logic; Section \ref{sec.comp} gives a detailed comparison with various versions of justification logic; Section \ref{sec.conc} concludes the paper with discussions and future directions.

\medskip

\section{A Logic of Knowing Why}
\label{sec.basic}
 \begin{defn}[Language $\ELKy$] Given a countable set $\A$ of agent names and a countably infinite set $\BP$ of basic propositional letters, the language of $\ELKy$ is defined as\jst   
$$\varphi::= p\mid \neg\varphi \mid (\varphi\wedge\varphi)\mid \Ki\varphi\mid \KWi\varphi$$
where $p\in \BP$ and $i\in \A$.
\end{defn}
We use standard abbreviations for $\top$, $\perp$, $\varphi\rightarrow\psi$, $\varphi\vee \psi$, and $\widehat{\Ki}\varphi$ (the dual of $\K_i\phi$). 
$\KWi\phi$ says that agent $i$ knows why $\phi$ (is the case). 


Intuitively, necessitation rule for $\KWi$ should not hold, e.g.,  although something is a tautology, you may not know why it is a tautology. Borrowing the idea from justification logic, we introduce a special set of ``self-evident'' tautologies which the agents are assumed to know why. Please see Section \ref{sec.comp} for the comparison with \textit{constant specifications} in justification logic, which may contain all axioms of the logic. 
\begin{defn}[Tautology Ground $\Lambda$]\label{kwl-tau-ground}
Tautology Ground $\Lambda$ is a set of propositional tautologies. 
\end{defn}
For example, $\Lambda$ can be the set of all the instances of $\varphi\land\psi\to \varphi$ and $\varphi\land \psi \to \psi$. As we will see later, under such a $\Lambda$, $\KWi(\varphi\land\psi\to \varphi)$ will be valid, which helps the agents to reason more. 


The model of our language $\ELKy$ is similar to the Fitting model of justification logic \cite{fitting2005logic}. Note that we do not have the justification terms in the logical language, but we do have a set $E$ of explanations as semantic objects in the models.  In this work, we require the accessibility relation to be equivalence relations to accommodate the S5 epistemic logic. 
\begin{defn}[$\ELKy$-Model]
An $\ELKy$-model $\M$ is a tuple 
\[
(W, E, \{R_i\mid i\in \A\},\E,V)
\]
where: 
\begin{itemize}
\item W is a non-empty set of possible worlds.
\item E is a non-empty set of explanations satisfying the following conditions\jst  
\begin{enumerate}[(a)]
\item If $s,t\in E$, then a new explanation $(s\cdot t)\in E$; 
\item A special symbol $e$ is in $E$. 
\end{enumerate}
\item $R_i\subseteq W\times W$ is an equivalence relation over $W$. 
\item $\E: E\times \ELKy \to 2^W$ is an \textit{admissible explanation function} satisfying the following conditions:
\begin{enumerate}[(I)]
\item $\E(s, \varphi \rightarrow \psi) \cap \E(t, \varphi) \subseteq \E(s\cdot t, \psi)$. 
\item If $\varphi\in \Lambda$, then $\E(e,\varphi)=W$.
\end{enumerate}
\item $V:\BP \to 2^W$ is a valuation function. 
\end{itemize}
\end{defn}
Note that $E$ does not depend on possible worlds, thus it can be viewed as a \textit{constant domain} of explanations closed under an \textit{application} operator~$\cdot$ which combines two explanations into one. The special element~$e$ in~$E$ is the \textit{self-evident explanation}, which is uniform for all the self-evident formulas in $\Lambda$. The admissible explanation function $\E$ specifies the set of worlds where  $t$ is an explanation of $\varphi$. It is possible that some formula has \textit{no} explanation at all on some world, and some formula has more than one explanation on some world, e.g., one theorem may have different proofs.\footnote{Though as Johan van Benthem pointed out via personal communication, in many cases there is often a best or shortest explanation for the fact.} The first condition of $\E$ captures the composition of explanations resembling the reasoning of knowing why by \textit{modus ponens}, which amounts to the later axiom $\KWi(\varphi\rightarrow \psi)\rightarrow(\KWi \varphi\rightarrow \KWi \psi)$.

\begin{defn}[Semantics]$\ $\\
The satisfaction relation of\/ $\ELKy$ formulas on pointed models is as below:

\begin{tabular}{|l@{\ }c@{\ }l|}
\hline 
$\mathcal{M}, w \vDash p $ & $\Longleftrightarrow$& $ w \in V(p)$ \\
$\mathcal{M}, w \vDash \neg\varphi$&$ \Longleftrightarrow $&$\mathcal{M}, w \not\vDash \varphi $\\
$\mathcal{M}, w \vDash \varphi \wedge \psi $&$\Longleftrightarrow $&$\mathcal{M}, w \vDash \varphi$ \ and \  $\mathcal{M}, w \vDash \psi$ \\
$\M, w\vDash \Ki\varphi$  & $\Longleftrightarrow$ & $\M, v\vDash \varphi$ \ for each $v$ such that $w R_i v$.\\
$\mathcal{M}, w\vDash \KWi\varphi$ & $\Longleftrightarrow$ &
 (1) $\exists t\in E$, for each $v$ such that $wR_i v$, $v\in \E(t,\varphi)$;\\
 &&(2) $\forall v\in W, wR_i v$, $v\vDash \varphi$.\\
 \hline
\end{tabular}
\end{defn}

Now it is clear that our $\KWi\varphi$ is roughly $\exists t \K_i({t\jst\varphi})\land \K_i\varphi$ though there are subtle details to be discussed in Section~\ref{sec.comp} when compared to justification logic. Also note that $\KWi\varphi\to \K_i\varphi$ is clearly valid, but $\KWi$-necessitation is not since not all valid formulas are explained except those in $\Lambda$. Moreover, things we usually take for granted are not valid either, e.g., $\KWi\varphi\land\KWi \psi\to \KWi(\varphi\land\psi)$ is not valid in general: The fact that I have explanations for $\varphi$ and $\psi$, respectively, does not mean that I have an explanation for the co-occurrence of the two, e.g., quantum mechanics and general relativity have their own explanatory power on microcosm and macrocosm, respectively, but a ``theory of everything'' is not obtained by simply putting these two theories together. 

As an example, in the following model (reflexive arrows are omitted), the formula $\K_i p\wedge \neg \KWi p\land \KWj p\land \K_i\KWj p$ holds on the middle world. 
$$
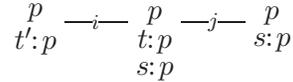
\xymatrix@R-10pt {
\txt{$p$\\$t'\jst p$}&\txt{{\phantom{p}}\\$p$\\$t\jst p$\\$s\jst p$}\ar@{-}[r]|j\ar@{-}[l]|i&\txt{$p$\\$s\jst p$}\\
}$$
In this paper, we also consider models with special properties. First of all, we are interested in the models where explanations are always correct,
i.e., if a proposition has an explanation on a world, then it must be true. 

\begin{defn}[Factivity Property] An $\ELKy$-model $\M$ has the factivity property provided that, whenever $w\in \E(t,\varphi)$, then $\M, w\vDash \varphi$.
\end{defn}
Philosophically this property is also debatable, but as we will see later in Theorem \ref{kwl-ci-cfi-equiv}, our logics stay neutral on it.\footnote{Martin Stokhof suggested  an example where one explains to the other why he is the best candidate for a job, but in fact he is not, and his explanation may base on false premises. } 

Besides factivity, it is also debatable whether knowing why is introspective, i.e., are the following reasonable? Note that they are not valid without further conditions on the models.
\begin{center}
$\K_i\varphi\to \KWi\K_i\varphi$, $\neg \K_i\varphi\to \KWi\neg \K_i\varphi$,\\ 
$\KWi\varphi \to\KWi\KWi\varphi$, $\neg \KWi\varphi \to \KWi\neg \KWi\varphi$
\end{center}
One may argue that there is always a self-evident explanation to your own knowledge or ignorance, but another may say it happens a lot that you just forgot why you know some facts. Things can be even more complicated regarding nested $\KWi$. Your explanation for why $\phi$ holds may be quite different from the explanation for why you know why $\phi$, e.g., the window is broken ($\phi$) because you know a stone was thrown at it, and you know why $\phi$ because someone told you so. On the other hand, if you know why a theorem holds because of a proof, it seems reasonable to assume that you know why you know why the theorem holds: you can just verify the proof. The cases of negative introspection may invoke more debates. 

As a first attempt to a logic of knowing why, we want to remain neutral in the philosophical debate, but would like to make it technically possible to handle the cases when introspection is considered reasonable. The following property guarantees that the above introspection axioms are valid.   
\begin{defn}[Introspection Property]
An $\ELKy$-model $\M$ has the introspection property provided that, whenever $\M, w\vDash\varphi$ and $\varphi$ has the form of $\Ki\psi$ or $\neg\Ki\psi$ or $\KWi\psi$ or $\neg\KWi\psi$, then $\exists t\in E$, for each $v$ such that $wR_i v$, $v\in \E(t, \varphi)$.
\end{defn}

We use $\bbc$, $\bbc_F$, $\bbc_I$,  $\bbc_{FI}$ to denote respectively the model classes of all $\ELKy$-models, factive models, introspective models, and models with both properties. Obviously, we have $\bbc_F\subseteq \bbc$, $\bbc_I\subseteq \bbc$, $\bbc_{FI}\subseteq \bbc_F$, and $\bbc_{FI}\subseteq \bbc_I$. In the following, we write $\Gamma\vDash_{\bbc} \varphi$ if $\M,w\vDash \Gamma$ implies \mbox{$\M,w\vDash \varphi$}, for any $\M\in \bbc$ and any $w$ in $\M$. Similar for $\bbc_F$, $\bbc_I$,  $\bbc_{FI}$.


However, as we will see below, factivity does not affect the valid formulas. For an arbitrary $\M\in \bbc$, we can construct a new $\ELKy$-model $\M^F\in \bbc_{F}$ which has factivtiy. Given $\M=(W, E, \{R_i\mid i\in \A\}, \E,V)$, let $\M^F=(W, E, \{R_i\mid i\in \A\}, \E^F, V)$ where:   
$$\E^F(t,\varphi)=\E(t,\varphi)-\{u\mid \M, u\not\vDash \varphi\}$$
We will show that $\M,w$ and $\M^F,w$ satisfy the same $\ELKy$ formulas, thus by the above definition of $\E^F$, it is clear that  $\M^F$ has factivity. 
\begin{lem}\label{kwl-c-cf-model-corres}
For any $\ELKy$-formula $\varphi$ and any $w\in W$, $\M, w\vDash\varphi$ if and only if $\M^F, w\vDash\varphi$.
\end{lem}
\begin{proof}
We can prove it by induction on the structure of formulas. It is trivial for the atomic, boolean, and $\K\psi$ cases since $\M^F$ only differs from $\M$ in $\E^F$. We just need to prove that $\M, w\vDash\KWi\psi$ iff $\M^F, w\vDash\KWi\psi$.
\begin{itemize}
\item[$\Rightarrow$] Suppose $\M, w\vDash\KWi\psi$. Then $\exists t\in E$, for each $v$ such that $wR_i v$, $v\in \E(t, \psi)$ and $v\vDash \psi$. Thus $v\not \in \{u\mid \M, u\not\vDash \psi\}$. Therefore we have $v\in \E^F(t, \psi)$. Hence by IH we have $\M^F, w\vDash \KWi\psi$.
\item[$\Leftarrow$] Suppose $\M^F, w\vDash\KWi\psi$. Then $\exists t\in E$, for each $v$ such that $wR_i v$, $v\in \E^F(t, \psi)$ and $v\vDash \psi$. By the definition of $\E^F$, we have $v\in \E(t,\psi)$. Hence by IH we get $\M, w\vDash \KWi\psi$. 
\end{itemize}
\end{proof}
\begin{thm}\label{thm.candcf}
For any set $\Gamma\cup\{\varphi\}$ of formulas, $\Gamma\vDash_\bbc \varphi$ if and only if $\Gamma \vDash_{\bbc_F}\varphi$.
\end{thm}
\begin{proof}
\begin{itemize}
\item[$\Rightarrow$] Suppose $\Gamma\vDash_\bbc \varphi$ and $\Gamma \not\vDash_{\bbc_F}\varphi$. By $\Gamma \not\vDash_{\bbc_F}\varphi$, there exists a factive model $\N\in \bbc_F$ such that $\N,w\vDash \Gamma$ and $\N,w\not\vDash \varphi$ for some $w$ in $\N$. Since $\bbc_F\subseteq \bbc$, we have $\N\in \bbc$. Thus $\Gamma\not\vDash_\bbc \varphi$. Contradiction.
\item[$\Leftarrow$] Suppose $\Gamma \vDash_{\bbc_F}\varphi$ and $\Gamma\not\vDash_\bbc \varphi$. Then there exists a model $\M\in \bbc$ such that $\M\vDash \Gamma$ and $\M\not\vDash \varphi$. By Lemma \ref{kwl-c-cf-model-corres}, we can construct an $\M^F\in \bbc_F$ such that $\M^F\vDash \Gamma$ and $\M^F\not\vDash \varphi$. Thus $\Gamma \not\vDash_{\bbc_F}\varphi$. Contradiction.
\end{itemize}
\end{proof}
Now we consider the introspective models. 
\begin{lem}\label{kwl-mfi-introspection}
If $\M$ is introspective, then so is $\M^{F}$.
\end{lem}
\begin{proof}
Suppose $\M^{F}, w\vDash\varphi$ and $\varphi$ has the form of $\Ki\psi$ or $\neg\Ki\psi$ or $\KWi\psi$ or $\neg\KWi\psi$. By Lemma \ref{kwl-c-cf-model-corres}, we have $\M, w\vDash\varphi$. Since $\M$ has introspection property, we have that $\exists t\in E$, for each $v$ such that $wR_i v$, $v\in \E(t, \varphi)$. Since $\M^{F}, w\vDash\varphi$ and $R_i$ is an equivalence relation, we have $\M^{F}, v\vDash \varphi$ for each $v$ such that $wR_i v$. Thus $v\not \in \{u\mid \M, u\not\vDash \varphi\}$. Thus $v\in \E^{F}(t,\varphi)$. Hence $\exists t\in E$, for each $v$ such that $wR_i v$, $v\in \E^{F}(t, \varphi)$. Therefore $\M^{F}$ has introspection property.
\end{proof}
It is then easy to show:
\begin{thm}\label{kwl-ci-cfi-equiv}
For any set $\Gamma\cup\{\varphi\}$, $\Gamma\vDash_{\bbc_I} \varphi$ if and only if $\Gamma \vDash_{\bbc_{FI}}\varphi$.
\end{thm}
Theorems \ref{thm.candcf} and \ref{kwl-ci-cfi-equiv} showed that factivity is neglectable w.r.t. the logic. 
\medskip

In the following, we present two proof systems which differ only on the introspection axioms of $\KWi$ essentially. In the next section, we will show their completeness w.r.t. $\bbc$ and $\bbc_{I}$, respectively.  
\begin{center}
System $\SKy$\\
\begin{minipage}{0.5\textwidth}
\begin{itemize}
\item[\TAUT] Classical Propositional Axioms
\item[\DISTK] $\Ki(\varphi\rightarrow \psi)\rightarrow(\Ki \varphi\rightarrow\Ki \psi)$
\item[\DISTY] $\KWi(\varphi\rightarrow \psi)\rightarrow(\KWi \varphi\rightarrow \KWi \psi)$
\item[\axT] $ \Ki\varphi\to\varphi$
\item[\transK] $\Ki \varphi\rightarrow\Ki \Ki\varphi$
\item[\eucK] $\neg \Ki \varphi\rightarrow\Ki \neg\Ki\varphi$
\item[\PRES] $\KWi \varphi\rightarrow\Ki \varphi$
\item[\transYK] $\KWi \varphi\rightarrow\Ki\KWi \varphi$
\end{itemize}
\end{minipage}
\begin{minipage}{0.36\textwidth}
\begin{itemize}
\item[\MP] Modus Ponens
\item[\NECK] $\vdash \varphi\ \ \Rightarrow \ \ \vdash \Ki\varphi$
\item[\NECY] If $\varphi\in \Lambda$, then $\vdash \KWi\varphi$ 
\end{itemize}
\end{minipage}
\end{center}
$\PRES$ is the presupposition axiom which says ``knowing that'' is necessary for ``knowing why''. \transYK\ is the positive introspection of ``knowing why'' by ``knowing that''.\footnote{Note that this is not one of the four introspection axioms of $\KWi$ mentioned earlier.} 
The reader may wonder about the corresponding negative introspection of $\transYK$ and it is provable in $\SKy.$
\begin{prop}\label{prop:5YK}
\eucYK: $\neg\KWi\varphi\rightarrow\Ki\neg\KWi\varphi$ is provable in $\SKy$.
\end{prop}
\begin{proof}$\ $\\
\begin{tabular}{cclcl}
(1) & \ \ &$\Ki\KWi\varphi\rightarrow \KWi\varphi$ &\ \ & \axT\\
(2) & \ \ & $\neg\KWi\varphi\rightarrow \neg\Ki\KWi\varphi$ & \ \ &Contraposition (1)\\
(3) & \ \ & $\neg\Ki\KWi\varphi\rightarrow \Ki\neg\Ki\KWi\varphi$ & \ \ & \eucK\\
(4) & \ \ & $\KWi\varphi\rightarrow\Ki\KWi\varphi$ & \ \ & \transYK\\
(5) & \ \ & $\neg\Ki\KWi\varphi\rightarrow\neg\KWi\varphi$ & \ \ & Contraposition (4)\\
(6) & \ \ & $\Ki(\neg\Ki\KWi\varphi\rightarrow\neg\KWi\varphi)$ & \ \ & \NECK (5)\\
(7) & \ \ & $\Ki\neg\Ki\KWi\varphi\rightarrow\Ki\neg\KWi\varphi$ & \ \ & \MP (6), \DISTK\\
(8) & \ \ & $\neg\KWi\varphi\rightarrow\Ki\neg\Ki\KWi\varphi$ & \ \ & \MP (2)(3)\\
(9) & \ \ & $\neg\KWi\varphi\rightarrow\Ki\neg\KWi\varphi$ & \ \ & \MP (7)(8)
\end{tabular}
\end{proof}
Note that the choice of $\Lambda$ and $\NECY$ in $\SKy$ also give us some flexibility in the logic. 

System $\SKyi$ is obtained by replacing \transK, \eucK\ and \transYK\ in $\SKy$ by the those four stronger introspection axioms of $\KWi$: 
\begin{center}
\begin{minipage}{0.39\textwidth}
\begin{itemize}
\item[\transKY] $\Ki\varphi\rightarrow \KWi\Ki\varphi$
\item[\eucKY] $\neg\Ki\varphi\rightarrow \KWi\neg\Ki\varphi$
\end{itemize}
\end{minipage}
\begin{minipage}{0.49\textwidth}
\begin{itemize}
\item[\transY] $\KWi \varphi\rightarrow\KWi\KWi \varphi$
\item[\eucY] $\neg\KWi\varphi\rightarrow\KWi\neg\KWi\varphi$  
\end{itemize}
\end{minipage}
\end{center}
It is straightforward to show that $\SKyi$ is deductively stronger than $\SKy$.  
\begin{prop}\label{kwlm-five-prop}
The following are provable in $\SKyi$
\begin{center}
\begin{minipage}{0.49\textwidth}
\begin{itemize}
\item[\transK] $\Ki\varphi\rightarrow \Ki\Ki\varphi$
\item[\eucK] $\neg\Ki\varphi\rightarrow \Ki\neg\Ki\varphi$
\end{itemize}
\end{minipage}
\begin{minipage}{0.49\textwidth}
\begin{itemize}
\item[\transYK]  $\KWi \varphi\rightarrow\Ki\KWi \varphi$
\item[\eucYK]  $\neg\KWi\varphi\rightarrow\Ki\neg\KWi\varphi$  
\end{itemize}
\end{minipage}
\end{center}
\end{prop}

\section{Soundness and Completeness}
\label{sec.SC}
Due to Theorems \ref{thm.candcf} and \ref{kwl-ci-cfi-equiv}, we only need to prove soundness and completeness w.r.t. $\bbc$ and $\bbc_{I}$ instead of $\bbc_F$ and $\bbc_{FI}$ respectively.

\begin{thm}[Soundness]
$\SKy$ and $\SKyi$ are sound for $\bbc$ and $\bbc_{I}$ respectively.
\end{thm}
\begin{proof}
Since $\ELKy$-models are based on S5 Kripke models, the standard axioms of system $\mathbb{S}5$ are all valid. So we just need to check the rest. First we check the non-trivial axioms and rules of $\SKy$ on $\bbc$.  
\begin{itemize}
\item[\DISTY :] $\KWi(\varphi\rightarrow \psi)\rightarrow(\KWi \varphi\rightarrow \KWi \psi)$\\
Suppose $w\vDash \KWi (\varphi\rightarrow\psi)$ and $w\vDash \KWi\psi$. Then by the definition of~$\vDash$, $\exists s,t \in E$, for any $v$ such that $wR_i v$, $v\in \E(s,\varphi\rightarrow\psi)$, 
\mbox{$v\in \E(t,\varphi)$}, \mbox{$v\vDash \varphi\rightarrow \psi$}, and $v\vDash \varphi$. We find
$v\vDash\psi$ and $v\in \E(s,\varphi\rightarrow\psi)\cap \E(t,\varphi)$. By the condition (I) of $\E$, we have $v\in \E(s\cdot t, \psi)$. Hence $w\vDash \KWi\psi$.
\item[\PRES :] $\KWi \varphi\rightarrow\Ki \varphi$\\
Suppose $w\vDash\KWi\varphi$. Then for any $v$ such that $wR_i v$, we have $v\vDash \varphi$. Thus $w\vDash \Ki\varphi$.
\item[\transYK:]$\KWi\varphi\to \Ki\KWi\varphi$\\ 
Suppose $w\vDash\KWi\varphi$. Then $\exists t \in E$, for any $v$ such that $wR_i v$, $v\in \E(t,\varphi)$ and $v\vDash \varphi$. Let $u \in W$ be arbitrary with $wR_i u$. Since $R_i$ is transitive, we find that
$u R_i v$ implies $w R_i v$. Thus $u\vDash\KWi\varphi$. We conclude that $w\vDash\Ki\KWi\varphi$.
\item[\NECY] Suppose $\varphi\in \Lambda$. Since $\Lambda$ is a set of tautologies, we have $\forall w\in W$, $w\vDash \varphi$. By the condition (II) of $\E$, $\forall w\in W, \exists e\in E$, for any $v$ such that $wR_i v$, $v\in \E(e,\varphi)$. Therefore it follows that $\vDash \KWi\varphi$. Hence \NECY\ is valid.
\end{itemize}
Validity of the introspection axioms of $\SKyi$ on $\bbc_{I}$ are trivial based on the introspective property and the fact that $R_i$ is an equivalence relation.  
\end{proof}
To establish completeness, we build a canonical model for each consistent set of $\ELKy$ formulas. We will first show the completeness of $\SKy$ over $\bbc$, and the completeness of $\SKyi$ over $\bbc_{I}$ is then straightforward. 

Let $\Omega$ be the set of all maximal $\SKy$-consistent sets of formulas. For any maximal consistent set (abbr. MCS) $\Gamma$, let $\Gamma^\#_i=\{\KWi\varphi\mid \KWi\varphi\in \Gamma\}\cup\{\varphi\mid \Ki\varphi\in\Gamma\}$. 

\begin{defn}[Canonical model for $\SKy$] The canonical model $\mathcal{M}^c$ for $\SKy$   is a tuple $ (W^c, E^c, \{R^c_i\mid i\in \A\}, \E^c, V^c)$ where:
\begin{itemize}
\item $E^c$ is defined in BNF: $t::=e\mid\varphi\mid (t\cdot t)$ where $\varphi\in \ELKy$.
\item $W^c=\{\lr{\Gamma, F, \{f_i\mid i\in \A\}}\mid\lr{\Gamma, F}\in \Omega\times\mathcal{P}( E^c\times \ELKy)$, $f_i: \{\varphi\mid \KWi\varphi\in \Gamma\}\to E^c$ such that $F$ and $\vv{f}$ satisfy the conditions below$\}$:
\begin{enumerate}[(i)]
\item If $\lr{s,\varphi\rightarrow\psi}, \lr{t,\varphi}\in F$, then $\lr{s\cdot t, \psi}\in F$; 
\item If $\varphi\in \Lambda$, then $\lr{e,\varphi}\in F$;
\item For any $i\in \A$, $\KWi\varphi\in \Gamma$ implies  $\lr{f_i(\varphi),\varphi}\in F$.
\end{enumerate}
\item $\lr{\Gamma, F, \vv{f} } R^c_i\lr{\Delta, G, \vv{g}}$ iff  (1) $\Gamma^\#_i\subseteq \Delta$, and (2) $f_i=g_i$.
\item $\E^c\jst  E^c\times \ELKy\rightarrow 2^{W^c}$ defined by $ \E^c(t, \varphi)=\{ \gff\mid \lr{t, \varphi}\in F\}$.
\item $V^c(p)=\{\gff\mid p\in \Gamma\}$.
\end{itemize}
\end{defn}
In the above we write $\vv{f}$ for $\{f_i\mid i\in \A\}$. Essentially, $f_i$ is a witness function picking one $t$ for each formula in $\{\varphi\mid \KWi\varphi\in \Gamma\}$. It can be used to construct the possible worlds for the existence lemma for $\neg \KWi\phi$. We do need such witness functions for each $i$, since $i,j$ can have different explanations for $\varphi$. In the definition of $R^c_i$, we need to make sure the selected witnesses are the same for $i$. We include  $\phi\in\ELKy$ as building blocks in $E^c$ for technical convenience, as it will become more clear below when we construct the successors. The component $F$ in each world is used to encode the information of $\E^c$ locally, also for the technical convenience to define the canonical relations. Note that merely maximal consistent sets are not enough in constructing the canonical model, as in the case of the logic of knowing what in \cite{WF13,WF14}.

Now we need to show that the canonical model is well-defined:
\begin{itemize}
\item $\E^c$ satisfies conditions (I) and (II) in the definition of $\ELKy$-models. 
\item $R^c_i$ is an equivalence relation.
\item $W^c$ is not empty. Actually, we will prove a stronger one: for any $\Gamma\in \Omega$, there exist $F$ and $\vv{f}$ such that $\gff\in W^c$.
\end{itemize}
\begin{prop}
$\E^c$ satisfies the conditions (I) and (II) of $\ELKy$-models. 
\end{prop}
\begin{proof}
\begin{enumerate}[(1)]
\item Suppose $\gff \in \E^c(s,\varphi\rightarrow\psi)\cap\E^c(t,\varphi)$. By the definition of $\ \E^c$, we have $\lr{s, \varphi\rightarrow\psi}, \lr{t, \varphi}\in F$. By the condition (i) of $F$ in the definition of $W^c$, we have $\lr{s\cdot t, \psi}\in F$. Hence it follows that $\gff\in \E^c(s\cdot t, \psi)$. Therefore $\E^c(s, \varphi \rightarrow \psi) \cap \E^c(t, \varphi) \subseteq \E^c(s\cdot t, \psi)$. 
\item Suppose $\varphi\in \Lambda$. For an arbitrary $\gff\in W^c$, by condition (ii) in the definition of $W^c$, we have $\lr{e,\varphi}\in F$ . By the definition of $\E^c$, we have $\gff\in \E^c(e, \varphi)$. Hence $\E^c(e, \varphi)=W^c$. 
\end{enumerate}
\end{proof}
Before proceeding further, we prove the following handy proposition.
\begin{prop}\label{kwlm-facilitate}
If $\gff R^c_i\dgg$, then (1) $\KWi\varphi\in \Gamma$ iff $\KWi\varphi\in \Delta$ and (2) $\Ki\varphi\in \Gamma$ iff $\Ki\varphi\in \Delta$.
\end{prop}
\begin{proof} $\ $
\begin{enumerate}
\item[(1)] Suppose $\KWi\varphi\in \Gamma$. By the definition of $R^c_i$, we have $\KWi\varphi\in \Delta$.
\item[] Suppose $\KWi\varphi\in \Delta$ and $\KWi\varphi\not\in \Gamma$. By the property of MCS, we have $\neg\KWi\varphi\in \Gamma$. By the provable \eucYK\ ($\neg\KWi\varphi\rightarrow\Ki\neg\KWi\varphi$) and the property of MCS, we have $\Ki\neg\KWi\varphi\in \Gamma$. By the definition of $R^c_i$, we have $\neg\KWi\varphi\in \Delta$. Contradiction. 
\item[(2)] Suppose $\Ki\varphi\in \Gamma$. By axiom \transK\ and the property of MCS, we have  $\Ki\Ki\varphi\in \Gamma$. By the definition of $R^c_i$, we have  $\Ki\varphi\in \Delta$.
\item[] Suppose $\Ki\varphi\in \Delta$ and $\Ki\varphi\not\in \Gamma$. By the property of MCS, we have  $\neg\Ki\varphi\in \Gamma$. By axiom $\eucK$ we have  $\Ki\neg\Ki\varphi\in \Gamma$. Then we have $\neg\Ki\varphi\in \Delta$ by the definition of $R^c_i$. Contradiction.
\end{enumerate}
\end{proof}
\begin{prop}
$R^c_i$ is an equivalence relation.
\end{prop}
\begin{proof}
We just need to prove $R^c_i$ is reflexive, transitive, and symmetric.
\begin{enumerate}[(1)]
\item $R^c_i$ is reflexive: For all $\Ki\psi\in \Gamma$, by axiom \axT\ we have $\psi\in \Gamma$. Hence we have $\gff R^c_i\gff$ by the definition of $R^c_i$.
\item $R^c_i$ is transitive: Suppose $\gff R^c_i\dgg $ and $\dgg  R^c_i \thhh $. Suppose $\KWi\varphi, \Ki\psi\in \Gamma$. By the definition of $R^c_i$, we have $f_i=g_i=h_i$. By Proposition \ref{kwlm-facilitate}, we have $\KWi\varphi, \Ki\psi\in \Delta$. By the definition of~$R^c_i$ we get $\KWi\varphi, \psi\in \Theta$. Therefore $\gff R^c_i \thhh $.
\item $R^c_i$ is symmetric: Suppose $\gff R^c_i\dgg  $. Then we have $f_i=g_i$. Suppose $\KWi\varphi,\Ki\psi\in \Delta $. By proposition \ref{kwlm-facilitate}, we have $\KWi\varphi\in \Gamma$ and $\Ki\psi\in \Gamma$. By axiom \axT, $\psi\in \Gamma$, thus $\dgg  R^c_i \gff$.
\end{enumerate}
\end{proof}

In order to establish that for any $\Gamma\in \Omega$, there exist $F$ and $\vv{f}$ such that $\gff\in W^c$, we define the following construction.
\begin{defn}
Given any $\Gamma\in\Omega$, construct $F^\Gamma$ and $\vv{f}^\Gamma$ as follows: 
\begin{itemize}
\item $F^\Gamma_0=\{\lr{\varphi,\varphi}\mid\exists i\in \A,\KWi\varphi\in \Gamma\}\cup\{\lr{e,\varphi}\mid \varphi\in \Lambda\}$
\item $F^\Gamma_{n+1}=F^\Gamma_n\cup\{\lr{s\cdot t,\psi}\mid\lr{s,\varphi\rightarrow\psi}\in F^\Gamma_n,\lr{t,\varphi}\in F^\Gamma_n \text{ for some $\varphi$}\}$
\item $F^\Gamma=\bigcup_{n\in \mathbb{N}}F^\Gamma_n$. 
\item $\forall i\in \A, f^\Gamma_i: \{\varphi\mid \KWi\varphi\in \Gamma\}\to E^c, f^\Gamma_i(\varphi)=\varphi$.
\end{itemize}

\end{defn}
By the construction of $F^\Gamma_n (n\in \mathbb{N})$, $\{F^\Gamma_n\mid n\in \mathbb{N}\}$ is monotonic. i.e., $\forall m,n\in N$, if $m\leqslant n$, then $F^\Gamma_m\subseteq F^\Gamma_n$.

\begin{prop}\label{kwlm-F-existence}
For any $\Gamma\in \Omega$, $\lr{\Gamma, F^\Gamma, \vv{f}^\Gamma}\in W^c$.
\end{prop}
\begin{proof}
 To prove $\lr{\Gamma, F^\Gamma, \vv{f}^\Gamma}\in W^c$, we just need to show that $F^\Gamma$ satisfies conditions (i)-(iii) in the definition of $W^c$. 
\noindent 
\begin{itemize}
\item Suppose $\lr{s,\varphi\rightarrow\psi}, \lr{t,\varphi}\in F^\Gamma$. By monotonicity of $\{F^\Gamma_n\mid n\in \mathbb{N}\}$, there exists $k\in \mathbb{N}$ such that $\lr{s,\varphi\rightarrow\psi}, \lr{t,\varphi}\in F^\Gamma_k$. Thus we have $\lr{s\cdot t, \psi}\in F^\Gamma_{k+1}$ by the construction of $F^\Gamma_k (k\in \mathbb{N})$. Hence $\lr{s\cdot t, \psi}\in F^\Gamma$, thus $F^\Gamma$ satisfies condition (i). 
\item Suppose $\varphi\in \Lambda$. By the construction of $F^\Gamma_0$, we have $\lr{e,\varphi}\in F^\Gamma_0$. Hence we get \mbox{$\lr{e,\varphi}\in F^\Gamma$}.  Thus $F$ satisfies condition (ii).
\item Suppose $\KWi\varphi\in \Gamma$. Then we have $\lr{\varphi,\varphi}\in F^\Gamma$ by the construction of $F^\Gamma_0$ and $F^\Gamma$. Since $\KWi\varphi\in \Gamma$, by the construction of $f^\Gamma_i$, we have $\varphi\in dom(f^\Gamma_i)$ and $f^\Gamma_i(\varphi)=\varphi$. Thus we have $\lr{f^\Gamma_i(\varphi),\varphi}\in F^\Gamma$. Hence, we have that $F^\Gamma$ and $\vv{f}^\Gamma$ satisfy condition (iii).
\end{itemize}
\end{proof}
This completes the proof that $\M^c$ is well-defined. Now we can establish the existence lemmas for $\K_i$ and $\KWi$. 
\begin{lem}[$\Ki$ Existence Lemma]\label{kwlm-k-existence-lemma}
For any $\gff\in W^c$, if $\widehat{\Ki}\varphi\in \Gamma$, then there exists a $\dgg \in W^c$ such that $\gff R^c_i\dgg $ and $\varphi\in \Delta$. 
\end{lem}
\begin{proof}
Suppose $\widehat{\Ki}\varphi\in \Gamma$. We will construct a $\dgg $ such that 
\[
\gff R^c \dgg \text{ and } \varphi\in \Delta.
\]
Let $\Delta^-$ be $\{\varphi\}\cup\{\KWi\psi\mid \KWi\psi\in \Gamma\}\cup\{\chi\mid\Ki\chi\in \Gamma\}$. Then $\Delta^-$ is consistent. Suppose not, then there are $\KWi\psi_1,\cdots,\KWi\psi_m, \chi_1,\cdots, \chi_n\in \Delta^-$ such that 
\[
\vdash_{\SKy}\KWi\psi_1\wedge\cdots\wedge\KWi\psi_m\wedge\chi_1\wedge\cdots\wedge \chi_n\rightarrow \neg\varphi.
\]
Then 
\[
\vdash_{\SKy}\Ki(\KWi\psi_1\wedge\cdots\wedge\KWi\psi_m\wedge\chi_1\wedge\cdots\wedge \chi_n)\rightarrow\Ki \neg\varphi.
\]
Since 
\begin{multline*}
\vdash_{\SKy}(\Ki\KWi\psi_1\wedge\cdots\wedge\Ki\KWi\psi_m\wedge\Ki\chi_1\wedge\cdots\wedge\Ki \chi_n)\rightarrow\\ \Ki(\KWi\psi_1\wedge\cdots\wedge\KWi\psi_m\wedge\chi_1\wedge\cdots\wedge \chi_n),
\end{multline*}
by propositional resoning, 
\[
\vdash_{\SKy}(\Ki\KWi\psi_1\wedge\cdots\wedge\Ki\KWi\psi_m\wedge\Ki\chi_1\wedge\cdots\wedge\Ki \chi_n)\rightarrow \Ki\neg\varphi.
\]
By $\KWi\psi_j\in \Gamma$ and axiom \transYK, we have  $\Ki\KWi\psi_j\in \Gamma$. Since $\Ki\chi_j\in \Gamma$, it follows that $\Ki\neg\varphi\in \Gamma$, i.e., $\neg\widehat{\Ki}\varphi\in \Gamma$. But this is impossible: $\Gamma$ is an MCS containing $\widehat{\Ki}\varphi$. We conclude that $\Delta^-$ is consistent.

Let $\Delta$ be any MCS containing $\Delta^-$, such extensions exist by a Lindenbaum-like argument. It follows that for any  $\KWi\varphi$, $\KWi\varphi\in \Gamma$ iff $\KWi\varphi\in \Delta$: 
\begin{itemize}
\item Suppose $\KWi\varphi\in \Gamma$. By the construction of $\Delta$, we have $\KWi\varphi\in \Delta$.
\item Suppose $\KWi\varphi\in \Delta$ and $\KWi\varphi\not\in \Gamma$. By the property of MCS, we have $\neg\KWi\varphi\in \Gamma$. By Proposition~\ref{prop:5YK}, we have $\Ki\neg\KWi\varphi\in \Gamma$. By the construction of $\Delta$, we have $\neg\KWi\varphi\in \Delta$. Contradiction.
\end{itemize}
In the following, we construct $G$ and $\vv{g}$ to form a world $\lr{\Delta, G, \vv{g}}$ in $W^c$. Based on the above result, we can simply let $g_i=f_i$.
We just need to construct $g_j$ for $j\not=i$. Formally, let:  
\begin{itemize}
\item $G_0=F\cup \{\lr{\varphi,\varphi}\mid \KWj\varphi\in \Delta \text{ for some }j\not=i\}$
\item $G_{n+1} = G_n\cup \{\lr{s\cdot t, \psi}\mid \lr{s,\varphi\rightarrow \psi}, \lr{t, \varphi}\in G_n\}$
\item $G=\bigcup_{n\in \mathbb{N}} G_n$
\end{itemize}
\[g_j(\varphi)=\left\{\begin{array}{ll}
f_j(\varphi)&\text{\qquad$j = i$},\\
\varphi & \text{\qquad $j\not =i$}.
\end{array}\right.\]
Since $F\subseteq G$ and $G$ is closed under implication, conditions (i) and (ii) are obvious. For condition (iii), if $\KWi\varphi\in \Delta$, then $\KWi\varphi\in \Gamma$. Thus 
\[
\lr{g_i(\varphi),\varphi}=\lr{f_i(\varphi),\varphi}\in F\subseteq G. 
\]
Condition (iii) also holds if $\KWj\varphi\in \Delta$ for $j\not=i$ by definition of $G_0$. It follows that $\dgg \in W^c$. By the construction of $\dgg$, we have $\varphi\in \Delta$, $\Gamma^\#_i\subseteq \Delta$, and $f_i=g_i$. Therefore there exists a state $\dgg \in W^c$ such that $\gff R^c_i\dgg $ and $\varphi\in \Delta$. 
\end{proof}

To refute $\KWi\psi$ semantically, for each explanation $t$ for $\psi$ at the current world, we need to construct an accessible world where $t$ is not an explanation for $\psi$. This leads to the following lemma. 

\begin{lem}[$\KWi$ Existence Lemma]\label{kwlm-kw-existence-lemma}
For any $\gff\in W^c$, if $\KWi\psi\not\in \Gamma$ then for any $\lr{t,\psi}\in F$, there exists $\dgg\in W^c$ such that $\lr{t,\psi}\not\in G$ and $\gff R^c_i \dgg$.
\end{lem}

\begin{proof}
Suppose $\KWi\psi\not\in \Gamma$, $\gff\in W^c$, and $\lr{t,\psi}\in F$. We construct $\dgg$ as follows: 
\begin{itemize}
\item $\Delta=\Gamma$
\item $\Psi=\{\lr{s,\varphi}\mid \lr{s,\varphi}\in F$ and $\KWi\varphi\not\in \Gamma\}$
\item $\Psi^\prime=\{\lr{t\cdot s,\varphi}\mid \lr{s,\varphi}\in\Psi\}$
\item $G_0= (F\backslash\Psi)\cup \Psi^\prime $
\item $G_{n+1}=G_n\cup \{\lr{r\cdot s, \varphi_2}\mid \lr{r,\varphi_1\rightarrow \varphi_2},\lr{s,\varphi_1}\in G_n\}$
\item $G=\bigcup_{n\in \mathbb{N}}G_n$ 
\item For each $j\in \A$, $g_j: \{\varphi\mid \KWj\varphi\in \Delta\}\to E^c$ is defined as: 
\begin{eqnarray*}
g_j(\varphi)=
\begin{cases}
f_j(\varphi), & \lr{f_j(\varphi),\varphi}\not\in \Psi \cr t\cdot f_j(\varphi), & \lr{f_j(\varphi),\varphi}\in \Psi 
\end{cases}
\end{eqnarray*}
\end{itemize}
Throughout this proof we write $\abs{s}>\abs{t}$ to express that $t$ is proper subterm of $s$.
From the construction of $G$, it is clear that 
for any $\lr{s,\varphi}\in \Psi^\prime$, we have $\abs{s}>\abs{t}$. We can show that for any $\KWi\phi\not\in\Gamma$, if  $\lr{s,\varphi}\in G_0$ then $\lr{s,\phi}\in \Psi'$. Towards contradiction, suppose that $\KWi\phi\not\in\Gamma$ and  $\lr{s,\phi}\in F\backslash\Psi$, then $\lr{s,\phi}\in F$, thus $\lr{s,\phi}\in \Psi$ by the definition of $\Psi$, contradiction.  It follows: 
\begin{equation}\label{small:G0:1}
\text{For any $\KWi\phi\not\in\Gamma$, if  $\lr{s,\varphi}\in G_0$ for some $s$, then $\abs{s}>\abs{t}$.}
\end{equation}

Thus in particular
\begin{equation}\label{eq:G0:1}
\lr{t, \psi} \text{ is not in } G_0.
\end{equation}

The idea behind the construction of $G$ is to first replace any current explanation for $\psi$ with something longer, and then take the closure w.r.t. implication. Note that for technical convenience, we treat all $\varphi$ such that $\KWi\varphi\not\in \Gamma$ in the basic step together. 

Now we prove the following claims. \\
\noindent\textbf{Claim 1} $\dgg\in W^c$. i.e., $G$ satisfies the conditions in the definition of $W^c$.
\begin{enumerate}[(i)]
\item Suppose $\lr{r, \varphi_1\rightarrow\varphi_2},\lr{s,\varphi_1}\in G$. By the construction of $G$, there exists $n\in \mathbb{N}$ such that $\lr{r, \varphi_1\rightarrow\varphi_2},\lr{s,\varphi_1}\in G_n$. By the construction of $G_{n+1}$, we have $\lr{r\cdot s,\varphi_2}\in G_{n+1}$. Thus $\lr{r\cdot s,\varphi_2}\in G$.
\item Suppose $\varphi\in \Lambda$. Then $\lr{e,\varphi}\in F$. Since $\varphi$ is a tautology, by \NECY\ and the property of MCS, we have $\KWi\varphi\in \Gamma$. Thus $\lr{e,\varphi}\not\in \Psi$. Thus $\lr{e,\varphi}\in G_0$. Hence $\lr{e,\varphi}\in G$.
\item Suppose $\KWj\varphi\in \Delta\ (j\in \A)$. By $\Delta=\Gamma$, we get $\KWj\varphi\in\Gamma$. Thus $\lr{f_j(\varphi),\varphi}\in F$. We have two cases:
\begin{itemize}
\item $ \lr{f_j(\varphi),\varphi}\not\in \Psi$: Thus $g_j(\varphi)=f_j(\varphi)$. Thus we have  $\lr{g_j(\varphi),\varphi}\in F$ and $ \lr{g_j(\varphi),\varphi}\not\in \Psi$. Thus  $\lr{g_j(\varphi),\varphi}\in G_0$. Hence $\lr{g_j(\varphi),\varphi}\in G$.
\item $ \lr{f_j(\varphi),\varphi}\in \Psi$: Thus $g_j(\varphi)=t\cdot f_j(\varphi)$ and $\lr{g_j(\varphi),\varphi}\in \Psi^\prime$. Thus we have  $\lr{g_j(\varphi),\varphi}\in G_0$. Hence $\lr{g_j(\varphi),\varphi}\in G$.
\end{itemize}
\end{enumerate}

\noindent\textbf{Claim 2} $\gff R^c_i \dgg$

To prove this claim, we just need to check two conditions:
\begin{enumerate}[(i)]
\item Since $\Delta=\Gamma$, obviously, we have $\Gamma^\#_i\subseteq \Delta$.
\item Since $\Delta=\Gamma$, we have $\{\varphi\mid \KWi\varphi\in \Gamma\}=\{\varphi\mid \KWi\varphi\in \Delta\}$, i.e., dom($g_i$)=dom($f_i$). For any $\varphi\in \{\varphi\mid \KWi\varphi\in \Delta\}$, since $\lr{f_i(\varphi),\varphi}\not\in \Psi$, by the definition of $g_i$, we have $g_i(\varphi)=f_i(\varphi)$. Hence $g_i=f_i$.
\end{enumerate}

To prove $\lr{t,\psi}\not\in G$, we first prove the following useful claim:\\
\noindent\textbf{Claim 3} If $\KWi\varphi\not\in \Gamma$ and $\lr{s,\varphi}\in G_{n+1}\backslash G_n$, then $\abs{s}>\abs{t}$.\\
Suppose $\KWi\varphi\not\in \Gamma$. Do induction on $n$:
\begin{itemize}
\item $n=0$. Suppose $\lr{s,\varphi}\in G_1\backslash G_0$. Then there exists $s_1, s_2$, and $\chi$ such that $s=s_1\cdot s_2$, $\lr{s_1,\chi\rightarrow\varphi},\lr{s_2,\chi}\in G_0$. We have two cases\jst 
\begin{itemize}
\item $\lr{s_1,\chi\rightarrow\varphi}\in \Psi^\prime$ or $\lr{s_2,\chi}\in \Psi^\prime$: Thus $\abs{s_1}>\abs{t}$ or $\abs{s_2}>\abs{t}$. Thus $\abs{s}>\abs{t}$.
\item $\lr{s_1,\chi\rightarrow\varphi}\not\in \Psi^\prime$ and $\lr{s_2,\chi}\not\in \Psi^\prime$: Since $\lr{s_1,\chi\rightarrow\varphi},\lr{s_2,\chi}\in G_0$, thus $\lr{s_1,\chi\rightarrow\varphi},\lr{s_2,\chi}\in F\backslash \Psi$. Thus $\KWi(\chi\rightarrow\varphi),\KWi\chi\in \Gamma$. Thus $\KWi\varphi\in \Gamma$ by axiom \DISTY. Contradiction.
\end{itemize}
\item $n > 0$. 
Suppose $\lr{s,\varphi}\in G_{n+1}\backslash G_n$. 
Then there exist $s_1,s_2,\chi$ such that $s=s_1\cdot s_2$ and 
$\lr{s_1,\chi\rightarrow \varphi},\lr{s_2,\chi}\in G_{n}$.
Moreover, we find that
\[
\KWi(\chi\rightarrow \varphi) \not\in \Gamma  \text{ or } \KWi\chi\not\in \Gamma
\]
since $\KWi\varphi\not\in \Gamma$ and $\Gamma$ is an MCS.
We also have 
\[
\lr{s_1,\chi\rightarrow\varphi}\not\in G_{n-1} \text{ or } \lr{s_2,\chi}\not\in G_{n-1}
\]
since otherwise $\lr{s,\varphi}\in G_n$ by the definition of $G_n$. Then we have  
\[
\lr{s_1,\chi\rightarrow \varphi}\in G_n\backslash G_{n-1} \text{ or }\lr{s_2,\chi}\in G_n\backslash G_{n-1}.
\]
We have the following cases:
\begin{itemize}
\item $\KWi(\chi\rightarrow \varphi) \not\in \Gamma$ and $\lr{s_1,\chi\rightarrow \varphi}\in G_n\backslash G_{n-1}$. By IH, we have $\abs{s_1}>\abs{t}$. Hence $\abs{s}>\abs{t}$.
\item $\KWi\chi\not\in \Gamma$ and $\lr{s_2,\chi}\in G_n\backslash G_{n-1}$. By IH, we have $\abs{s_2}>\abs{t}$. Hence $\abs{s}>\abs{t}$.
\item $\KWi(\chi\rightarrow \varphi) \not\in \Gamma$ and $\lr{s_1,\chi\rightarrow \varphi}\in G_{n-1}$. If $\lr{s_1,\chi\rightarrow \varphi}\in G_0$, then we have $\abs{s_1}>\abs{t}$ by \eqref{small:G0:1}; If  $\lr{s_1,\chi\rightarrow \varphi}\not\in G_0$, then there exists $0 < k<n$ such that $\lr{s_1,\chi\rightarrow \varphi}\in G_k\backslash G_{k-1}$. Thus by IH we have $\abs{s_1}>\abs{t}$. Thus $\abs{s}>\abs{t}$.
\item $\KWi\chi\not\in \Gamma$ and $\lr{s_2,\chi}\in G_{n-1}$. Similar to the above.
\end{itemize}

\end{itemize}

\noindent\textbf{Claim 4} $\lr{t,\psi}\not\in G$. 

According to the construction of $G$, we just need to show that for all \mbox{$n\in \mathbb{N}$}, $\lr{t,\psi}\not\in G_n$. 
By~\eqref{eq:G0:1}, we already know $\lr{t,\psi}\not\in G_0$.
Based on Claim~3, $\lr{t,\psi}$ cannot be added in any $G_n$ for $n\geq 1.$ 
We conclude  $\lr{t,\psi}\not\in G$. 
\end{proof}

Finally we are ready to prove the truth lemma.
\begin{lem}[Truth Lemma]\label{kwlm-truth-lemma}
For all $\varphi$, $\gff \vDash\varphi$  if and only if $ \varphi\in \Gamma$.
\end{lem}
\begin{proof}
This is established by standard induction on the complexity of~$\varphi$. The atomic cases and the boolean cases are standard. The case when  $\varphi=\Ki\psi$ is also routine based on Lemma \ref{kwlm-k-existence-lemma}.


Consider the case that $\varphi$ is $\KWi\psi$ for some $\psi$.
\begin{itemize}
\item[$\Leftarrow$] If $\KWi\psi\in \Gamma$, then for any $\dgg$ such that $\gff  R^c_i\dgg $, we have then $ \KWi\psi\in \Delta$ by the definition of $R^c_i$. Since $\vdash_{\SKy}\KWi\psi\rightarrow\psi$ (by \axT\ and \PRES), we have $\psi\in \Delta$.  By IH, we have \mbox{$\dgg \vDash\psi$}. Since $\KWi\psi\in \Gamma$ and  $\KWi\psi\in \Delta$, we have $\lr{f_i(\psi),\psi}\in F$ and $\lr{g_i(\psi),\psi}\in G$. By the definition of $R^c_i$, we have $f_i=g_i$. Thus there exists $g_i(\psi)=f_i(\psi)\in E^c$ such that $\dgg\in \E^c(g_i(\psi),\psi)$. Therefore we conclude $\gff \vDash \KWi\psi$.
\item[$\Rightarrow$] Suppose $\KWi\psi\not\in \Gamma$. We have two cases as follows: 
\begin{itemize}
\item $\K_i\psi\not\in \Gamma$: then by Lemma \ref{kwlm-k-existence-lemma} and the semantics, $\gff\not\vDash \KWi\psi$.
\item $\K\psi\in \Gamma$: We also have two cases:
\begin{itemize}
\item $\lr{t,\psi}\not\in F$ for all $t\in E$. By the semantics, $\gff\not\vDash \KWi\psi$.
\item There exists $t\in E$ such that $\lr{t,\psi}\in F$. By Lemma \ref{kwlm-kw-existence-lemma}, there exists $\dgg\in W^c$ with $\lr{t,\psi}\not\in G$ and $\gff R^c_i \dgg$. Hence we have $\gff\nvDash \KWi\psi$.
\end{itemize}
\end{itemize}
\end{itemize}
\end{proof}
\begin{thm}[Completeness of $\SKy$ over $\bbc$]\label{kwlm-c-completeness}
$\Sigma\vDash_{\bbc} \varphi$ implies $\Sigma\vdash_{\SKy} \varphi$.
\end{thm}
\begin{proof}
Suppose $\Sigma\vDash_{\bbc} \varphi$. Towards a contradiction, suppose $\Sigma\not\vdash_{\SKy} \varphi$. Then $\Sigma\cup \{\neg\varphi\}$ is consistent. Extend $\Sigma\cup\{\neg\varphi\}$ to a maximal consistent set $\Gamma$. By Proposition \ref{kwlm-F-existence}, there exist $F$ and $\vv{f}$ such that $\gff \in W^c$. By Lemma \ref{kwlm-truth-lemma}, we have $\gff \vDash\Sigma\cup \{\neg\varphi\}$, thus $\Sigma\cup \{\neg\varphi\}$ is satisfiable, thus $\Sigma\vDash_{\bbc} \varphi$ is false. Contradiction. 
\end{proof}

By Theorem \ref{thm.candcf} and Theorem \ref{kwlm-c-completeness}, we have the following corollary.
\begin{corol}[Completeness of $\SKy$ over $\bbc_{F}$]
$\Sigma\vDash_{\bbc_{F}} \varphi$ implies \mbox{$\Sigma\vdash_{\SKy} \varphi$}.
\end{corol}

Now let us look at the completeness of $\SKyi$. The crucial observation is that we can use the same canonical model definition except now we let $\Omega$ be the set of all maximal $\SKyi$-consistent set of $\ELKy$ formulas. The similar propositions follow due to Proposition \ref{kwlm-five-prop}. The only extra thing is to check whether the new canonical model has the introspection property. 
\begin{prop}
$\M^c$ has introspection property.
\end{prop}
\begin{proof}
Suppose $\gff\vDash\varphi$ and $\varphi$ has the form of $\Ki\psi$ or $\neg\Ki\psi$ or $\KWi\psi$ or $\neg\KWi\psi$. By Lemma \ref{kwlm-truth-lemma}, we have $\varphi\in \Gamma$. By the axioms \transKY-\eucY\ and the properties of MCS, we have $\KWi\varphi\in \Gamma$. By Lemma \ref{kwlm-truth-lemma}, we have $\gff\vDash \KWi\varphi$. Thus $\exists r\in E^c$, $\dgg\in \E^c(r,\varphi)$ for each $\dgg$ such that $\gff R^c_i\dgg$.
\end{proof}

Based on the above proposition and Theorem \ref{kwl-ci-cfi-equiv} we have:
\begin{thm}[Completeness of $\SKyi$ over $\bbc_{I}$ and $\bbc_{FI}$]\label{kwlm-ci-completeness}$\ $\\
If $\Sigma\vDash_{\bbc_{I}} \varphi$, then \mbox{$\Sigma\vdash_{\SKyi} \varphi$}. If $\Sigma\vDash_{\bbc_{FI}} \varphi$, then $\Sigma\vdash_{\SKyi}\varphi$.
\end{thm}

\section{Comparison with justification logic}
\label{sec.comp}

In this section, we compare our framework with justification logic. We first explain our deviations from the standard justification logic, and then give an alternative semantics of our logic $\SKy$, which is technically closer to the standard setting of justification logic.
\subsection{Similarities and differences}
The language of the most classic justification logic $\LP$ (i.e., $\mathbf{JT4}$ in \cite{artemov2008logic}) includes both formulas $\varphi$ and justification terms $t$: 
\begin{center}
$\varphi::= p\mid \neg\varphi \mid (\varphi\wedge\varphi)\mid {t\jst\varphi}$\\
$t::= x\mid c \mid (t\cdot t)\mid (t+t) \mid {!t}$
\end{center}

The possible-world semantics of justification logic is based on the Fitting model $\lr{S, R, \E, V}$ where $\lr{S, R, V}$ is a single-agent Kripke model and $\E$ is an \textit{evidence function} assigning justification terms $t$ to formulas on each world, just as in our setting. The formula $t\jst\varphi$ has the following semantics (cf. e.g., \cite{Fitting16}):
\begin{center}
\begin{tabular}{|lcl|}
\hline
$\M, w\Vdash t\jst\varphi$  & $\Longleftrightarrow$ &(a) $w\in \E(t, \varphi)$;\\
 &&(b) $v\Vdash \varphi$ for all $v$ such that $wRv$.\\
\hline 
\end{tabular}
\end{center}
Compared to our semantics for $\KWi\phi$, note that (a) only requires that $t$ is a justification of $\phi$ on the current world $w$. The Fitting models for $\LP$ are assumed to have further conditions:\footnote{The ``S5 version'' of justification logic $\mathbf{JT45}$ also adds another condition about negative introspection: $\overline{\E(t, \varphi)} \subseteq \E(?t, \neg (t\jst\varphi))$, and requires strong evidence, where $?$ is a new operation for justification terms in the language, cf. \cite{artemov2008logic}. To simplify the discussion, we focus on $\LP$ here.}
\begin{enumerate}
\item[(1)] $\E(s, \varphi \rightarrow \psi) \cap \E(t, \varphi) \subseteq \E(s\cdot t, \psi)$
\item[(2)] $\E(t,\varphi)\cup\E(s,\varphi) \subseteq \E(s+t,\varphi)$
\item[(3)] $\E(t,\varphi)\subseteq \E(!t, t\jst\varphi)$
\item[(4)] Monotonicity: $w\in\E(t,\varphi)$ and $wRv$ implies $v\in \E(t,\varphi).$
\item[(5)] $R$ is reflexive and transitive. 
\end{enumerate}
Note that we also require (1) and (5) above and include $\cdot$ as an operation on explanations in $E$ semantically. On the other hand, we leave out (2)(3)(4) and the operations $+$ and $!$ for specific considerations in our setting. For the case of $+$, consider the following  model where $\varphi$ has two possible explanations and agent $i$ cannot distinguish them (thus $\neg \KWi\varphi$ holds). 
$$
\xymatrix@R-15pt{
t\jst\varphi \ar@{-}[r]|i& s\jst \varphi\\
}
$$
If we impose condition (2) then $s+t$ is a uniform explanation of $\varphi$ on both worlds, which makes $\KWi\varphi$ true. More generally, for any finite model where $\varphi$ has some explanations on each world, $\KWi\varphi$ will always be true under condition (2), which is counterintuitive in our setting. Conceptually, the explanation should be \textit{precise}, you cannot explain a theorem by saying one of all the possible proofs up to a certain length works. Knowing there is a proof does not mean you know why the theorem holds.

Operation $!$ and conditions (3) and (4) are relevant to the validity of the axiom $t\jst\varphi\to {!t\jst(t\jst\varphi)}$ in the justification logic $\LP$, which is used to realize axiom \texttt{4} in modal logic. Intuitively, $!$ is the proof checker and $!t$ will always be a justification of $t\jst\phi$.\footnote{In the multi-agent setting, $!_i$ was introduced to capture the proof check done by each agent \cite{Yavorskaya2006}.} Although we do not have $t\jst\varphi$ in the language, it may sound reasonable to include $!$ and require $\E(t,\varphi)\subseteq \E(!t, \KWi\varphi)$. 
However, $t$ being an explanation
for $\varphi$ does not entail that $t$ can be transformed uniformly into an explanation for $\KWi\varphi$. For example, the window is broken since someone threw a rock at it, but there can be different explanations for an agent to know why the window is broken: she saw it, or someone told her about it, and so on. 

The technically motivated condition (4) in justification logic requires that 
any accessible possible world has more explanations than the actual world,
which is not reasonable in our setting: an undesired consequence of condition~(4) would be
$w\in \E(t,\varphi)$ and $w\vDash\K_i\varphi$ imply $w\vDash\KWi \varphi$.   


There are justification logics available with both $\Box \varphi$ and $t\jst\varphi$, see, e.g.,~\cite{ArtemovN05,KuStAiML12}. Justification terms are used to represent explicit knowledge whereas the $\Box$-operator is used for implicit knowledge. Hence these logics feature the principle   
\begin{equation}\label{eq:jyb:1}
t\jst\varphi \to \Box \varphi,
\end{equation}
which is based on the idea that one may implicitly know more than what is explicitly justified. 
Note that in the presence of the $!$-operation, the principle~\eqref{eq:jyb:1} implies
\begin{equation}\label{eq:jyb:2}
t\jst\varphi \to \Box {t\jst \varphi}.
\end{equation}
Indeed, by~\eqref{eq:jyb:1} we have 
$!t \jst  (t \jst \phi) \to \Box ( t \jst \phi)$, which
together with the axiom
$t \jst \phi \to {!t \jst  (t \jst \phi)}$ yields \eqref{eq:jyb:2}.
The formulas~\eqref{eq:jyb:1} and~\eqref{eq:jyb:2} correspond in our setting to the axioms
$\KW_i\varphi\to \K_i\varphi$ and $\KW_i\varphi\to \K_i \KW_i \varphi$, respectively.

In some versions of multi-agent justification logic, e.g.~\cite{BucKuzStu11JANCL,Renne2012}, the evidence function $\E$ is agent-dependent (or, equivalently, each agent has her own justifications), and correspondingly the formula $t \!\!:_i\! \phi$ is introduced into the language to express that $t$ is a justification of $\phi$ for $i$. However, in our models, we use a uniform function $\E$ for all agents since we think the explanatory relation between explanations and formulas is also part of the possible worlds, just like basic propositional facts, which are interpreted by a valuation function independent from the agents.   

Justification logics are parameterized by a constant specification (CS), a collection of $c \jst \varphi$ formulas where $c$ is a justification constant and $\phi$ is an axiom of the justification logic. 
It controls which axioms the logic provides justifications for, i.e.~which axioms an agent may use in her justified reasoning process.
A justification logic model meets the requirement of a given CS if $W=\E(c,\varphi)$ for all $c \jst \varphi\in \text{CS}$. 
In contrast, we include only tautologies (but not all axioms) in our tautology ground $\Lambda$. For example, if we had $(\K_i\varphi\to \varphi)\in\Lambda$, then we could derive $\KWi(\K_i\varphi\to \varphi)$ by \NECY, which would imply $\KWi\K_i\varphi \to \KWi\varphi$ by \DISTY. That, however, would be a strange consequence: e.g., I know why I know that the window is broken implies that I know why it is broken. 

  The table below highlights the similarities between our axioms (or derivable theorems in $\SKy$ and $\SKyi$) and axioms in (variants of) justification logic when viewing $t\jst\phi$ as $\KWi\phi$: 
\begin{center}
\begin{tabular}{|@{\quad}c@{\quad}|@{\quad}c@{\quad}|}
\hline 
 Justification Logic & Our work \\ 
\hline 
$t\jst(\varphi\to\psi)\to s\jst \varphi\to (t\cdot s)\jst \psi$ & $\KWi(\varphi\rightarrow \psi)\rightarrow(\KWi \varphi\rightarrow\KWi \psi)$  \\ 
\hline 
$t\jst\varphi \to (s+t)\jst\varphi$ & $\KWi\varphi\to\KWi\varphi$ \\ 
\hline 
$t\jst\varphi\to \varphi$ & $\KWi\varphi\to \varphi$ \\ 
\hline 
 $t\jst\varphi\to !t\jst(t\jst\varphi)$ & $\KWi\varphi\rightarrow \KWi\KWi\varphi$\\ 
\hline 
 $\neg t\jst\varphi\to ?t\jst(\neg {t\jst\varphi})$ & $\neg\KWi\varphi\rightarrow \KWi\neg\KWi\varphi$\\ 
\hline 
 $t\jst\varphi\to \Box\varphi$ \cite{ArtemovN05}& $\KWi \varphi\rightarrow\Ki \varphi$\\ 
\hline 
$t\jst\varphi\to \Box t\jst\varphi$ \cite{ArtemovN05} & $\KWi\varphi\to \Ki\KWi\varphi$ \\ 
\hline
$\neg t\jst\varphi\to \Box \neg t\jst\varphi$ \cite{ArtemovN05} & $\neg \KWi\varphi\to \Ki\neg \KWi\varphi$ \\ 
\hline 
\end{tabular} 
\end{center}

To close this comparison, note that Fitting proposed a quantified justification logic in \cite{Fitting08}, and discussed briefly in the end what can expressed if the language also includes the normal knowledge operator. Since $\KWi$ implicitly includes quantification over explanations, our language can then be viewed as a fragment of this quantified justification logic extended with $\K$.
\subsection{An alternative semantics}
The similarities between our work and justification logic make it technically possible to give a more standard justification logic semantics to  $\ELKy$-formulas.
In the following we evaluate formulas over multi-agent Fitting models, see, e.g., \cite{BucKuzStu11JANCL,Renne2012}, where each agent has her own accessibility relation and evidence function.\footnote{The alternative semantics does not work if we just have only one evidence function.} The interpretation of \emph{agent $i$ knows why $\phi$} is given as \emph{agent $i$ knows that $\phi$ and has some justification of $\phi$}, that is \emph{knowing why} translates to \emph{having a justification}.

\begin{defn}[$\JL$-Model]
A $\JL$-model $\M$ is a tuple 
\[
(W, E, \{R_i\mid i\in \A\}, \{\E_i\mid i\in \A\},V)
\]
where: 
\begin{itemize}
\item W is a non-empty set of possible worlds.
\item E is a non-empty set of explanations satisfying the following conditions\jst  
\begin{enumerate}[(a)]
\item If $s,t\in E$, then a new explanation $(s\cdot t)\in E$; 
\item A special symbol $e$ is in $E$. 
\end{enumerate}
\item $R_i\subseteq W\times W$ is an equivalence relation over $W$. 
\item $\E_i: E\times \ELKy \to 2^W$ is an \textit{admissible evidence function} satisfying the following conditions:
\begin{enumerate}[(I)]
\item $\E_i(s, \varphi \rightarrow \psi) \cap \E_i(t, \varphi) \subseteq \E_i(s\cdot t, \psi)$. 
\item If $\varphi\in \Lambda$, then $\E_i(e,\varphi)=W$.
\item Monotonicity: $w \in \E_i(t, \phi)$ and $wR_iv$ implies $v \in \E_i(t, \phi)$.
\end{enumerate}
\item $V:\BP \to 2^W$ is a valuation function. 
\end{itemize}
\end{defn}
\begin{remark}\label{rem.s5}
Note that by imposing monotonicity on S5 models, all $i$-indistinguishable worlds have the same justifications for the same formula, i.e., if $w R_i v$ then 
\[
w\in \E_i(t, \phi) \text{ iff } v\in \E_i(t, \phi).
\]
\end{remark}

\begin{defn}[Semantics]$\ $\\
The satisfaction relation of $\ELKy$-formulas on pointed $\JL$-models is as below:

\begin{tabular}{|l@{\ }c@{\ }l|}
\hline 
$\mathcal{M}, w \vDashJ p $ & $\Longleftrightarrow$& $ w \in V(p)$ \\
$\mathcal{M}, w \vDashJ \neg\varphi$&$ \Longleftrightarrow $&$\mathcal{M}, w \not\vDashJ \varphi $\\
$\mathcal{M}, w \vDashJ \varphi \wedge \psi $&$\Longleftrightarrow $&$\mathcal{M}, w \vDashJ \varphi$ and $\  \mathcal{M}, w \vDashJ \psi$ \\
$\M, w\vDashJ \Ki\varphi$  & $\Longleftrightarrow$ & $\M, v\vDashJ \varphi$ for each $v$ such that $w R_i v$.\\
$\mathcal{M}, w\vDashJ \KWi\varphi$ & $\Longleftrightarrow$ &
 (1) $\exists t\in E$ such that $w\in \E_i(t,\varphi)$;\\
 &&(2) $\forall v\in W, wR_i v$ implies $\mathcal{M}, v\vDashJ \varphi$.\\
 \hline
\end{tabular}
\end{defn}
Compared to our semantics, the crucial difference in the above semantics of $\KWi\phi$ is that it only requires that $t$ is a justification on the \textit{current} world $w$.

\begin{thm}[Soundness]
$\SKy$ is sound for $\JL$-models.
\end{thm}
\begin{proof}
Since $\JL$-models are based on S5 Kripke models, the standard axioms of system $\mathbb{S}5$ are all valid. So we just need to check the rest.  
\begin{itemize}
\item[\DISTY :] $\KWi(\varphi\rightarrow \psi)\rightarrow(\KWi \varphi\rightarrow \KWi \psi)$\\
Suppose $w\vDashJ \KWi (\varphi\rightarrow\psi)$ and $w\vDashJ \KWi\varphi$. 
By soundness of $\DISTK$, we obtain $\forall v\in W, wR_i v$ implies $\mathcal{M}, v\vDashJ \psi$.
Further, by the definition of $\vDashJ$, there exist $s,t \in E$ with $w\in \E_i(s,\varphi\rightarrow\psi)$ and $w\in \E_i(t,\varphi)$.
By the closure conditions on admissible evidence functions we get
$w\in \E_i(s\cdot t, \psi)$. Hence $w\vDashJ \KWi\psi$.
\item[\PRES :] $\KWi \varphi\rightarrow\Ki \varphi$\\
Follows immediately from the definition of $\vDashJ$.
\item[\transYK:]$\KWi\varphi\to \Ki\KWi\varphi$\\ 
Suppose $w\vDashJ \KWi\phi$ and let $v \in W$ be arbitrary with $wR_iv$. By transitivity of $R_i$ we find that
$\forall u\in W, vR_i u$ implies 
$\mathcal{M}, u\vDashJ \varphi$.
Further there exists~$t$ with 
$w \in \E_i(t,\varphi)$ and
monotonicity of 
$\E_i$ yields
$v\in \E_i(t,\varphi)$. We obtain $v \vDashJ \KWi\varphi$ and conclude
$w \vDashJ \Ki\KWi\varphi$.
\item[\NECY] 
Suppose $\varphi\in \Lambda$. By condition (II) on $\E_i$, we get $w \in \E_i(e,\varphi)$ for any $w$.
Since $\Lambda$ is a set of tautologies, we also have that $wR_i v$ implies $v\vDashJ \varphi$. 
Hence $\KWi \phi$ is valid.
\end{itemize}
\end{proof}

To establish completeness of $\SKy$ with respect to $\JL$-models, we show how to transform a given $\ELKy$-model into an equivalent $\JL$-model. Then completeness for $\JL$-models is a consequence of completeness for $\ELKy$-models.

\begin{defn}[Corresponding $\JL$-model]
Given an $\ELKy$-model 
\[
\M = (W, E, \{R_i\mid i\in \A\},\E,V)
\]
we define the corresponding $\JL$-model $\MJ$ as 
\[
(W, E, \{R_i\mid i\in \A\}, \{\EJ_i\mid i\in \A\},V)
\]
where
\[
\EJ_i(t,\varphi) := \{w \mid \forall v \in W,\ w R_i v \text{ implies } v \in \E(t,\varphi) \}. 
\]
\end{defn}
The above definition indeed yields a $\JL$-model.
Moreover, any given $\ELKy$-model~$\M$ and its corresponding $\JL$-model~$\MJ$ satisfy the same formulas. We have the following two lemmas.
\begin{lem}
Let $\M$ be an $\ELKy$-model, then $\MJ$ is a $\JL$-model.
\end{lem}
\begin{proof}
We have to verify the conditions on $\EJ$:
\begin{enumerate}
\item[(I)] Suppose $w \in \EJ_i(s, \varphi \rightarrow \psi) \cap \EJ_i(t, \varphi)$. Thus for each $v \in W$, $w R_i v$ implies 
\[
v \in \E(s, \varphi \rightarrow \psi) \cap \E(t, \varphi)
\]
and hence also 
$v \in \E(s\cdot t, \psi)$. By the definition of $\EJ_i$ we conclude $w \in \EJ_i(s\cdot t, \psi)$.
\item[(II)] If $\varphi \in \Lambda$, then $\E(e,\varphi)=W$ and hence also $\EJ_i(e,\varphi)=W$.
\item[(III)] Assume $w \in \EJ_i(t, \phi)$ and $w R_i v$.
Let $u$ be arbitrary with $v R_i u$. Since $R_i$ is transitive, we get $w R_i u$. Thus by the definition of $\EJ_i$ we find
$u \in \E(t, \phi)$. We conclude $v\in \EJ_i(t, \phi)$.
\end{enumerate}
\end{proof}
\begin{lem}
Let $\M = (W, E, \{R_i\mid i\in \A\},\E,V)$ be an $\ELKy$-model. For each $w \in W$ and each formula $\varphi$,
\[
\M,w \vDash \varphi \text{ if and only if } \MJ,w \vDashJ \varphi .
\]
\end{lem}
\begin{proof}
By induction on $\varphi$.\\
Case $\varphi$ is of the form $\KWi\psi$.
Observe that by the definition of $\EJ_i$ we have that
\[
\text{$\exists t\in E$, for each $v$ such that $wR_i v$, $v\in \E(t,\psi)$}
\]
if and only if 
\[
\text{$\exists t\in E$ with $w\in \EJ_i(t,\psi)$}.
\]
Hence $\M, w\vDash \KWi\psi$ if and only if $\MJ, w\vDashJ \KWi\psi$.\\
All other cases are trivial.
\end{proof}
\begin{remark}
The above result also holds if we consider S4-based $\ELKy$-models and $\JL$-models, i.e., when the relations $R_i$ are only reflexive and transitive.  
\end{remark}
\begin{corol}[Completeness]
$\SKy$ is strongly complete for $\JL$-models.
\end{corol}
\begin{proof}
Suppose $\Gamma\not\vdash \varphi$.
By completeness with respect to $\ELKy$-models, 
there is an $\ELKy$-model~$\M$ with a world $w$ such that $ \M,w\vDash\Gamma$ but $\M,w \nVdash \varphi$.
By the previous two lemmas, we find a $\JL$-model $\MJ$ such that 
$\MJ,w\vDashJ \Gamma$ but 
$\MJ,w \not \vDashJ \varphi$.
\end{proof}

We consider the above results about $\JL$-models to be merely technical observations. From a conceptual point of view, as we mentioned, monotonicity in $\JL$-models is not very reasonable in our setting. $\JL$-models are in fact rather weak for our purpose. This can be seen from the fact that a lot of information is lost in the translation from $\ELKy$-models to $\JL$-model, in particular $\ELKy$-models only store known explanations but all other possible explanations are dropped.
Hence $\JL$-model cannot really talk about the difference between $\exists x\K E(x, \varphi)$ and $\K \exists x E(x, \varphi)$, which is essential for our analysis of \emph{knowing why}. Moreover, this $\JL$-like semantics cannot handle conditional knowledge-why, as will be introduced formally in the next section. For example, it is reasonable to have a situation where I don't know why $\phi$ right now, but I know why $\phi$ \textit{given} the information $\psi$,  since the extra information of $\psi$ may rule out some possibilities such that there is a uniform explanation on the remaining possibilities. Due to Remark \ref{rem.s5}, this is not possible in a monotonic S5 (or S4) $\JL$-model with monotonicity. 
\section{Conclusions and Future work}
\label{sec.conc}
In this paper, we present an attempt to formalize the logic of knowing why. In the language we have both the standard knowing that operator~$\K_i$ and the new knowing why operator~$\KWi$. A semantics based on Fitting-like models for justification logic is given, which  interprets knowing why $\phi$ as \textit{there exists an explanation such that I know it is one explanation for $\phi$}. We gave two proof systems, one weaker and one stronger depending on the choice of introspection axioms, and showed their completeness over various model classes. 

Note that, in the logic of knowing value \cite{WF13,WF14}, there is one and only one value for each constant. However, there can be different explanations for the same fact in our setting. This difference also leads to some technical complications in the completeness proof: to negate $\KWi\phi$, it is not enough to just construct another   world. Instead, for \textit{each} explanation $t$ of $\phi$ at the current world, we need to construct one world to refute it.  

As the title shows, it is by no means \textit{the} logic of knowing why. Besides the introspection axioms, there are a lot to be discussed.\footnote{We may also discuss whether $\KWi\phi\to\KWi\Ki\phi$ is reasonable.} For example, although \DISTY\ looks reasonable in a setting focusing on deductive explanations, it may cause troubles if causal explanations or other types of explanations are considered. Recall our example about the flagpole and its shadow. It is reasonable to assume that I know why the shadow is $y$ meters long ($\KWi p$), and I also know why that the shadow is $y$~meters implies the pole is $x$ meters long ($\KWi(p\to q)$). However, it does not entail that I know why the pole is $x$ meters long ($\KWi q$) if we are looking for causal explanation (or functional explanation). One way to go around is to replace the material implication by some other relevant (causal)  conditional, then $\KWi(p\to q)$ may not hold in this setting anymore. 

It seems that we often do not have clear semantic intuition about non-trivial expressions of knowing why. One reason is that there may be different readings of the same statement of knowing why $\phi$  regarding different aspects of $\phi$ and different types of desired explanations. For example, ``I know why Frank went to Beijing on Monday'' may have different meanings depending on the \textit{contrast} the speaker wants to emphasize \cite{van1980scientific}:  
\begin{itemize}
\item I know why \textit{Frank}, not Mary, went to Beijing on Monday.  
\item I know why Frank went to \textit{Beijing}, not Shanghai, on Monday. \item I know why Frank went to Beijing on \textit{Monday},  not on Tuesday. 
\end{itemize}
Following \cite{koura1988approach}, we may partially handle this by adding contrast formulas, e.g., turn $\KWi\varphi$ into $\KWi(\varphi\land \neg \psi\land \neg \chi\land \dots)$ depending on the emphasis. However, we cannot handle the changes of types of explanations depending on the contrast. 


Another future direction is to study the inner structure of explanations further. Hintikka's early work \cite{hintikka1995semantics} may turn out to be helpful, where explanations can be of the form of universally quantified formulas, which connects better with the existing theories of scientific explanations in philosophy of science.\footnote{There are also modal logic approaches to handle scientific explanations cf. e.g, \cite{Seselja2013,Igor}.} Moreover, we may be interested in saying whether an explanation is true. The factivity that we proposed did not fully capture that. 

A promising future study is about group notions of knowing why. For example, how do we define everyone knows why $\varphi$? Simply having a conjunction of $\KWi\varphi$ for each $i$ may not be enough, since people can have different explanations for $\varphi$. The case of \textit{commonly} knowing why $\varphi$ is more interesting. For example, we may have different definitions:
\begin{itemize}
\item It is (standard) common knowledge that everyone knows why $\varphi$ w.r.t. the same explanation.
\item Everyone knows why \dots everyone knows why $\varphi$. 
\end{itemize}
In contrast to standard epistemic logic, such definitions can be quite different from each other. Since each iteration of $\KWi$ may ask for a new explanation, we then have a much richer spectrum of such common knowledge notions, e.g., for the second definition, we may ask the agents to have exactly the same explanation for each level of ``iteration of everyone knows why''. It will be interesting to compare such notions with justified common knowledge~\cite{Artemove06,BucKuzStu11JANCL}.

Of course, we can also consider the dynamics of knowing why, similar to the dynamics in justification logic \cite{Bucheli14,ks13,Renne:2008,Renne2012}. Clearly, public announcements can change knowledge-why. However, in contrast to public announcement logic \cite{PubPlazanew}, adding public announcement will increase the expressive power of the logic, e.g., $[q]\KWi p$ can distinguish the following two pointed models (the left-hand-side worlds as designated), which cannot be distinguished by formulas in $\ELKy$ (a simple inductive proof on the structure of the formula suffices):
$$
\xymatrix@R-10pt{
\txt{$p,q$\\$s\jst p$}\ar@{-}[r]^i&\txt{$p$\\$t\jst p$}
}
\qquad \qquad
\xymatrix@R-10pt{
\txt{$p,q$\\$s\jst p$}\ar@{-}[r]^i\ar@{-}[d]^i&\txt{$p$\\$t\jst p$}\\
\txt{$p,q$\\$r\jst p$}
}
$$
In particular, $[q]\KWi p$ is not equivalent to $q\to \KWi(q\to p)$. To handle public announcements, we can follow the idea in \cite{WF13} and generalize the knowing why operator to a conditional one to express that the agent $i$ knows why $\phi$ given the condition $\psi$ ($\KWi(\psi, \phi)$):
\begin{center}
\begin{tabular}{|cl|}
\hline
$\M, w\vDash \KWi(\psi, \varphi) \Longleftrightarrow$ & $\exists t\in E$, for each $v$ such that $wR_i v$ and $\M, v\vDash\psi$:\\
&(1) $v\in \E(t,\varphi)$ and (2) $\M, v\vDash \varphi$.\\
\hline
\end{tabular}
\end{center}
$\KWi(\psi, \phi)$ is similar to $[\psi]\KWi \phi$, and may be  used to encode the announcements under further restrictions on models. We leave the axiomatization of this more expressive language to further work.  

On the other hand, there can be other natural dynamics, e.g., publicly announcing why, which is similar to public inspection introduced in \cite{GvEW16} in the setting of knowing values. A deeper connection between knowing why and dynamic epistemic logic is possible based on the observation that we do update according to events because we know why they happened (the preconditions). It is suggested by Olivier Roy that there is also a close connection with forward induction in games, where it is crucial to guess why someone did an apparently irrational move. 





Finally, our work is also related to \textit{explicit knowledge}, which aims to avoid logical omniscience. In fact, knowledge with justification or explanation can be viewed as a type of explicit knowledge. One important approach to define explicit knowledge is by using \textit{awareness}: $\varphi$ is a piece of explicit knowledge of $i$ $(X_i\varphi)$ if $i$ is aware of $\varphi$ $(A_i\varphi)$ and $i$ implicitly knows that $\varphi$ $(\K_i\varphi)$, where awareness is often defined syntactically (cf.~\cite{RAK}). Accordingly, the axioms are also changed, e.g., the \texttt{K} axiom  now becomes $X_i(\varphi\to \psi) \land X_i\varphi\land A_i\psi\to X_i\psi.$ Other approaches to explicit knowledge uses idea of \textit{algorithmic knowledge} \cite{HP11:LogOm}. We may explore the concrete connection in the future. 



\bibliographystyle{spmpsci}
\bibliography{fkw}

\begin{thebibliography}{10}
\providecommand{\url}[1]{{#1}}
\providecommand{\urlprefix}{URL }
\expandafter\ifx\csname urlstyle\endcsname\relax
  \providecommand{\doi}[1]{DOI~\discretionary{}{}{}#1}\else
  \providecommand{\doi}{DOI~\discretionary{}{}{}\begingroup
  \urlstyle{rm}\Url}\fi

\bibitem{Artemov95}
Artemov, S.: Operational modal logic (1995).
\newblock Technical Report MSI 95–29, Cornell University

\bibitem{Art01BSL}
Artemov, S.: Explicit provability and constructive semantics.
\newblock Bulletin of Symbolic Logic \textbf{7}(1), 1--36 (2001)

\bibitem{Artemove06}
Artemov, S.: Justified common knowledge.
\newblock Theoretical Computer Science \textbf{357}(1), 4 -- 22 (2006)

\bibitem{artemov2008logic}
Artemov, S.: The logic of justification.
\newblock The Review of Symbolic Logic \textbf{1}(04), 477--513 (2008)

\bibitem{ArtemovN05}
Art{\"{e}}mov, S.N., Nogina, E.: Introducing justification into epistemic
  logic.
\newblock Journal of Logic and Computation \textbf{15}(6), 1059--1073 (2005)

\bibitem{vB91}
van Benthem, J.: Reflections on epistemic logic.
\newblock Logique and Analyse \textbf{34}(133-134), 5--14

\bibitem{bird98}
Bird, A.: Philosophy of science.
\newblock Routledge (1998)

\bibitem{bromberger1966questions}
Bromberger, S.: Questions.
\newblock The Journal of Philosophy \textbf{63}(20), 597--606 (1966)

\bibitem{BucKuzStu11JANCL}
Bucheli, S., Kuznets, R., Studer, T.: Justifications for common knowledge.
\newblock Journal of Applied Non-classical Logics \textbf{21}(1), 35--60
  (2011).
\newblock \doi{10.3166/JANCL.21.35-60}

\bibitem{Bucheli14}
Bucheli, S., Kuznets, R., Studer, T.: Realizing public announcements by
  justifications.
\newblock Journal of Computer and System Sciences \textbf{80}(6), 1046--1066
  (2014)

\bibitem{HBEL}
van Ditmarsch, H., Halpern, J.Y., van~der Hoek, W., Kooi, B. (eds.): Handbook
  of Epistemic Logic.
\newblock College Publications (2015)

\bibitem{RAK}
Fagin, R., Halpern, J., Moses, Y., Vardi, M.: Reasoning about knowledge.
\newblock MIT Press (1995)

\bibitem{fitting2005logic}
Fitting, M.: The logic of proofs, semantically.
\newblock Annals of Pure and Applied Logic \textbf{132}(1), 1--25 (2005)

\bibitem{Fitting08}
Fitting, M.: A quantified logic of evidence.
\newblock Annals of Pure and Applied Logic \textbf{152}(1), 67 -- 83 (2008)

\bibitem{Fitting16}
Fitting, M.: Modal logics, justification logics, and realization.
\newblock Annals of Pure and Applied Logic \textbf{167}(8), 615--648 (2016)

\bibitem{van1980scientific}
van Fraassen, B.C.: The scientific image.
\newblock Oxford University Press (1980)

\bibitem{GvEW16}
Gattinger, M., van Eijck, J., Wang, Y.: Knowing values and public inspection.
\newblock In: Proceedings of ICLA '17 (2017).
\newblock Forthcoming

\bibitem{HP11:LogOm}
Halpern, J.Y., Pucella, R.: Dealing with logical omniscience: Expressiveness
  and pragmatics.
\newblock Artificial Intelligence \textbf{175}(1), 220--235 (2011)

\bibitem{hempel1965aspects}
Hempel, C.: Aspects of Scientific Explanation and Other Essays in the
  Philosophy of Science.
\newblock The Free Press (1965)

\bibitem{hempel1948studies}
Hempel, C.G., Oppenheim, P.: Studies in the logic of explanation.
\newblock Philosophy of science \textbf{15}(2), 135--175 (1948)

\bibitem{hintikka1962knowledge}
Hintikka, J.: Knowledge and belief: an introduction to the logic of the two
  notions, vol. 181.
\newblock Cornell University Press Ithaca (1962)

\bibitem{hintikka1981logic}
Hintikka, J.: On the logic of an interrogative model of scientific inquiry.
\newblock Synthese \textbf{47}(1), 69--83 (1981)

\bibitem{Hintikka1983}
Hintikka, J.: New Foundations for a Theory of Questions and Answers, pp.
  159--190.
\newblock Springer Netherlands, Dordrecht (1983)

\bibitem{hintikka1995semantics}
Hintikka, J., Halonen, I.: Semantics and pragmatics for why-questions.
\newblock The Journal of Philosophy \textbf{92}(12), 636--657 (1995)

\bibitem{kitcher1981explanatory}
Kitcher, P.: Explanatory unification.
\newblock Philosophy of Science \textbf{48}(4), 507--531 (1981)

\bibitem{koura1988approach}
Koura, A.: An approach to why-questions.
\newblock Synthese \textbf{74}(2), 191--206 (1988)

\bibitem{KuStAiML12}
Kuznets, R., Studer, T.: Justifications, ontology, and conservativity.
\newblock In: T.~Bolander, T.~Bra{\"u}ner, S.~Ghilardi, L.~Moss (eds.) Advances
  in Modal Logic, volume 9, pp. 437--458. College Publications (2012)

\bibitem{ks13}
Kuznets, R., Studer, T.: Update as evidence: Belief expansion.
\newblock In: S.~Artemov, A.~Nerode (eds.) Logical Foundations of Computer
  Science, LFCS~2013, \emph{LNCS}, vol. 7734, pp. 266--279. Springer (2013).
\newblock \doi{10.1007/978-3-642-35722-0_19}

\bibitem{PubPlazanew}
Plaza, J.: Logics of public communications.
\newblock Synthese \textbf{158}(2), 165--179 (2007)

\bibitem{Renne:2008}
Renne, B.: Dynamic epistemic logic with justification.
\newblock Ph.D. thesis, New York, NY, USA (2008).
\newblock AAI3310607

\bibitem{Renne2012}
Renne, B.: Multi-agent justification logic: communication and evidence
  elimination.
\newblock Synthese \textbf{185}(1), 43--82 (2012)

\bibitem{salmon1984scientific}
Salmon, W.: Scientific Explanation and the Causal Structure of the World.
\newblock Princeton University Press (1984)

\bibitem{Schurz1995}
Schurz, G.: Scientific explanation: A critical survey.
\newblock Foundations of Science \textbf{1}(3), 429--465 (1995)

\bibitem{Schurz1999}
Schurz, G.: Explanation as unification.
\newblock Synthese \textbf{120}(1), 95--114 (1999)

\bibitem{schurz2005explanations}
Schurz, G.: Explanations in science and the logic of why-questions: Discussion
  of the \textrm{Halonen--Hintikka}-approachand alternative proposal.
\newblock Synthese \textbf{143}(1), 149--178 (2005)

\bibitem{Igor}
Sedl\'ar, I., Halas, J.: Modal logics of abstract explanation frameworks
  (2015).
\newblock Abstract in Proceedings of CLMPS 15

\bibitem{Seselja2013}
{\v{S}}e{\v{s}}elja, D., Stra{\ss}er, C.: Abstract argumentation and
  explanation applied to scientific debates.
\newblock Synthese \textbf{190}(12), 2195--2217 (2013)

\bibitem{Wang15:lori}
Wang, Y.: A logic of knowing how.
\newblock In: Proceedings of LORI-V, pp. 392--405 (2015)

\bibitem{Wang16}
Wang, Y.: Beyond knowing that: a new generation of epistemic logics.
\newblock In: H.~van Ditmarsch, G.~Sandu (eds.) Jaakko Hintikka on knowledge
  and game theoretical semantics. Springer (2016).
\newblock Forthcoming

\bibitem{WF13}
Wang, Y., Fan, J.: Knowing that, knowing what, and public communication: Public
  announcement logic with \texttt{Kv} operators.
\newblock In: Proceedings of IJCAI'13, pp. 1139--1146 (2013)

\bibitem{WF14}
Wang, Y., Fan, J.: Conditionally knowing what.
\newblock In: Proceedings of AiML Vol.10, pp. 569--587 (2014)

\bibitem{weber2013scientific}
Weber, E., van Bouwel, J., De~Vreese, L.: Scientific explanation.
\newblock Springer (2013)

\bibitem{Yavorskaya2006}
Yavorskaya~(Sidon), T.: Multi-agent explicit knowledge.
\newblock In: D.~Grigoriev, J.~Harrison, E.A. Hirsch (eds.) Proceedings of CSR
  2006, pp. 369--380. Springer (2006)

\end{thebibliography}

\end{document}